\documentclass{article}

\usepackage{fullpage}

\usepackage[utf8]{inputenc}
\usepackage[T1]{fontenc}
\usepackage{textcomp}
\usepackage{bm}
\usepackage{dsfont}

\usepackage{amsmath}
\usepackage{amssymb}
\usepackage{amsthm}
\usepackage{amsfonts}       

\usepackage{mathtools}
\DeclarePairedDelimiter{\prn}{(}{)}
\DeclarePairedDelimiter{\set}{\{}{\}}
\DeclarePairedDelimiterX{\Set}[2]{\{}{\}}{\,{#1}\,:\,{#2}\,}
\DeclarePairedDelimiter{\abs}{|}{|}
\DeclarePairedDelimiter{\norm}{\|}{\|}
\DeclarePairedDelimiter{\inpr}{\langle}{\rangle}

\DeclarePairedDelimiter{\brc}{[}{]}
\DeclarePairedDelimiterX{\Brc}[2]{[}{]}{\,{#1}\,\middle|\,{#2}\,}
\DeclarePairedDelimiter{\ceil}{\lceil}{\rceil}

\DeclareFontFamily{U}{mathx}{}
\DeclareFontShape{U}{mathx}{m}{n}{<-> mathx10}{}
\DeclareSymbolFont{mathx}{U}{mathx}{m}{n}
\DeclareMathAccent{\widecheck}{0}{mathx}{"71}

\usepackage[svgnames]{xcolor}

\usepackage{booktabs}
\usepackage{multirow}

\makeatletter
\let\blx@noerroretextools\@empty
\makeatother
\usepackage[date=year,eprint=false,doi=false,isbn=false,maxcitenames=2,natbib,url=false,sorting=nyt,sortcites=ynt,style=trad-abbrv,backref=true]{biblatex}

\DeclareSourcemap{
    \maps[datatype=bibtex]{
        \map{
            \step[fieldset=editor, null]
            \step[fieldset=series, null]
            \step[fieldset=language, null]
            \step[fieldset=address, null]
            \step[fieldset=location, null]
            \step[fieldset=month, null]
        }
    }
}
\usepackage{algorithm}

\usepackage[colorlinks=true,citecolor=Navy,linkcolor=Maroon,urlcolor=Orchid,bookmarksnumbered,hypertexnames=false,pdfdisplaydoctitle,pdfusetitle,unicode]{hyperref}

\usepackage[noend]{algpseudocode}

\algnewcommand{\algorithmicinput}{\textbf{Input:}}
\algnewcommand{\Input}{\item[\algorithmicinput]}
\algnewcommand{\algorithmicoutput}{\textbf{Input:}}
\algnewcommand{\Output}{\item[\algorithmicoutput]}
\algnewcommand{\Break}{\textbf{break}}

\usepackage{url}
\usepackage{upref}
\usepackage[capitalize,noabbrev]{cleveref}

\crefname{step}{Step}{Steps}
\Crefname{step}{Step}{Steps}
\crefname{appendix}{Appendix}{Appendices}
\Crefname{appendix}{Appendix}{Appendices}

\usepackage{autonum}

\usepackage{nicefrac}       
\usepackage{microtype}      
\usepackage{xspace}
\usepackage{subcaption}
\usepackage{wrapfig}
\usepackage{siunitx}
\usepackage{thm-restate}
\allowdisplaybreaks
\usepackage{enumitem}

\usepackage[color={red!100!green!33},colorinlistoftodos,prependcaption,textsize=small]{todonotes}

\usepackage{pgfplots}
\usepackage{tikz, tikz-3dplot}

\usetikzlibrary{positioning}
\usetikzlibrary{arrows}

\newtheorem{theorem}{Theorem}[section]
\newtheorem{lemma}[theorem]{Lemma}
\newtheorem{proposition}[theorem]{Proposition}
\newtheorem{corollary}[theorem]{Corollary}
\newtheorem{assumption}[theorem]{Assumption}
\theoremstyle{definition}
\newtheorem{definition}[theorem]{Definition}

\newtheorem{remark}[theorem]{Remark}

\newcommand{\R}{\mathbb{R}}

\DeclareMathOperator{\argmin}{arg\,min}
\DeclareMathOperator{\argmax}{arg\,max}

\newcommand{\E}{\mathop{\mathbb{E}}}

\newcommand{\Wcal}{\mathcal{W}}

\newcommand{\Ccal}{\mathcal{C}}
\newcommand{\Ncal}{\mathcal{N}}

\newcommand{\fmg}{f}
\newcommand{\qons}{q}

\newcommand{\Rtl}{\tilde{R}}

\title{Online Inverse Linear Optimization: Efficient Logarithmic-Regret Algorithm, Robustness to Suboptimality, and Lower Bound}

\author{%
Shinsaku Sakaue\footnotemark[1]\\
CyberAgent\\
\href{mailto:shinsaku.sakaue@gmail.com}{shinsaku.sakaue@gmail.com}
\and
Taira Tsuchiya\\
The University of Tokyo and RIKEN AIP\\
\href{mailto:tsuchiya@mist.i.u-tokyo.ac.jp}{tsuchiya@mist.i.u-tokyo.ac.jp}
\and
Han Bao\footnotemark[1]\\
Institute of Statistical Mathematics\\
\href{mailto:bao.han@ism.ac.jp}{bao.han@ism.ac.jp}
\and
Taihei Oki\\
Hokkaido University\\
\href{mailto:oki@icredd.hokudai.ac.jp}{oki@icredd.hokudai.ac.jp}
}

\date{}

\begin{document}

\makeatletter
\def\thefootnote{*}
\makeatother
\footnotetext[1]{Most of this work was done while SS was at the University of Tokyo and RIKEN AIP and HB was at Kyoto University.}
\makeatletter
\def\thefootnote{\arabic{footnote}}
\makeatother

\maketitle

\begin{abstract}
  In online inverse linear optimization, a learner observes time-varying sets of feasible actions and an agent's optimal actions, selected by solving linear optimization over the feasible actions.
  The learner sequentially makes predictions of the agent's true linear objective function, and their quality is measured by the \emph{regret}, the cumulative gap between optimal objective values and those achieved by following the learner's predictions. 
  A seminal work by B{\"a}rmann et al.~(2017) obtained a regret bound of $O(\sqrt{T})$, where $T$ is the time horizon.
  Subsequently, the regret bound has been improved to $O(n^4 \ln T)$ by Besbes et al.~(2021, 2025) and to $O(n \ln T)$ by Gollapudi et al.~(2021), where $n$ is the dimension of the ambient space of objective vectors. 
  However, these logarithmic-regret methods are highly inefficient when $T$ is large, as they need to maintain regions specified by $O(T)$ constraints, which represent possible locations of the true objective vector. 
  In this paper, we present the first logarithmic-regret method whose per-round complexity is independent of $T$; 
  indeed, it achieves the best-known bound of $O(n \ln T)$. 
  Our method is strikingly simple: it applies the online Newton step (ONS) to appropriate exp-concave loss functions. 
  Moreover, for the case where the agent's actions are possibly suboptimal, we establish a regret bound of $O(n\ln T + \sqrt{\Delta_T n\ln T})$, where $\Delta_T$ is the cumulative suboptimality of the agent's actions.
  This bound is achieved by using MetaGrad, which runs ONS with $\Theta(\ln T)$ different learning rates in parallel. 
  We also present a lower bound of $\Omega(n)$, showing that the $O(n\ln T)$ bound is tight up to an $O(\ln T)$ factor.  
\end{abstract}

\section{Introduction}\label{sec:introduction}
Optimization problems serve as forward models of various processes and systems, ranging from human decision-making to natural phenomena. 
In real-world applications, the true objective function of such models is rarely known a priori. 
This motivates the problem of inferring the objective function from observed optimal solutions, or \emph{inverse optimization}.
Early work in this area emerged from geophysics, aiming at estimating subsurface structure from seismic wave data \citep{Tarantola1988-tq,Burton1992-dc}. 
Subsequently, inverse optimization has been extensively studied \citep{Ahuja2001-cv,Heuberger2004-zv,Chan2019-zg,Chan2023-qk}, applied across various domains, such as transportation \citep{Bertsimas2015-kw}, power systems \citep{Birge2017-il}, and healthcare \citep{Chan2022-uq}, and have laid the foundation for various machine learning methods, including inverse reinforcement learning \citep{Ng2000-sf} and contrastive learning \citep{Shi2023-nd}.\looseness=-1

This study focuses on an elementary yet fundamental case where the objective function of forward optimization is linear. 
We consider an \emph{agent} who repeatedly selects an action from a set of feasible actions by solving forward linear optimization.\footnote{\label{footnote:linear-model} An ``agent'' is sometimes called an ``expert,'' which we do not use to avoid confusion with the expert in universal online learning (see \cref{subsec:metagrad}). 
Additionally, our results could potentially be extended to nonlinear settings based on kernel inverse optimization \citep{Bertsimas2015-kw,Long2024-gl}, although we focus on the linear setting for simplicity.\looseness=-1}
Let $n$ be a positive integer and $\R^n$ the ambient space where forward optimization is defined.
For $t=1,\dots,T$, given a set $X_t \subseteq \R^n$ of feasible actions, the agent selects an action $x_t \in X_t$ that maximizes $x \mapsto \inpr{c^*, x}$ over $X_t$, where $c^* \in \R^n$ is the agent's internal objective vector and $\inpr{\cdot, \cdot}$ denotes the standard inner product on~$\R^n$. 
We want to infer~$c^*$ from observations consisting of the feasible sets and the agent's optimal actions, i.e., $\set{(X_t, x_t)}_{t=1}^T$.\looseness=-1

For this problem, \citet{Barmann2017-wl,Barmann2020-hh} have shown that online learning methods are effective for inferring the agent's underlying linear objective function.
In their setting, for $t=1,\dots,T$, a \emph{learner} makes a prediction $\hat c_t$ of $c^*$ based on the past observations $\set{(X_i, x_i)}_{i=1}^{t-1}$ and receives $(X_t, x_t)$ as feedback.
Let $\hat x_t \in \argmax_{x \in X_t}\inpr{\hat c_t, x}$ represent an optimal action induced by the learner's $t$th prediction.
Their idea is to regard $\R^n \ni c \mapsto \inpr{c, \hat x_t - x_t}$ as a cost function and apply online learning methods, such as the online gradient descent (OGD).
Then, the standard regret analysis ensures that $\sum_{t=1}^T \inpr{\hat c_t - c, \hat x_t - x_t}$ grows at the rate of $O(\sqrt{T})$ for any~$c$. 
Letting $c = c^*$, this bound also applies to $\sum_{t=1}^T \inpr{c^*, x_t - \hat x_t}$, which is the \emph{regret} incurred by choosing $\hat x_t$ based on the learner's prediction~$\hat c_t$, since $\inpr{\hat c_t, \hat x_t - x_t}$ is non-negative due to the optimality of $\hat x_t$ for $\hat c_t$. 
As such, online learning methods with sublinear regret bounds can make the average regret converge to zero as $T\to\infty$.

While the regret bound of $O(\sqrt{T})$ is optimal in general online linear optimization (e.g., \citet[Section~3.2]{Hazan2023-pq}), the above online inverse linear optimization has special problem structures that could allow for better regret bounds; intuitively, since $x_t \in X_t$ is optimal for $c^*$, feedback $(X_t, x_t)$ is more informative about $c^*$ defining the regret.
\Citet{Besbes2021-ak,Besbes2023-zm} indeed showed that a logarithmic regret bound of $O(n^4\ln T)$ is possible, and \citet{Gollapudi2021-ad} further improved the bound to $O(n\ln T)$.\footnote{
  \citet{Gollapudi2021-ad} studied the same problem under the name of \emph{contextual recommendation}.
}
While these methods significantly improve the dependence on $T$ in the regret bounds, their per-round computation cost is prohibitively high when $T$ is large.
Specifically, these methods iteratively update (appropriately inflated) regions that indicate possible locations of the true objective vector $c^*$ and set prediction $\hat c_t$ to the ``center'' of the regions (the circumcenter in \citet{Besbes2021-ak,Besbes2023-zm} and the centroid in \citet{Gollapudi2021-ad}).
Since those regions are represented by $O(T)$ constraints, their per-round complexity, at least in a straightforward implementation, grows polynomially in $T$. 
Indeed, \citet{Besbes2021-ak,Besbes2023-zm} and \citet{Gollapudi2021-ad} only claim that their methods run in $\mathrm{poly}(n, T)$ time. 
This is in stark contrast to the online-learning approach of \citet{Barmann2017-wl,Barmann2020-hh}, whose per-round complexity is independent of $T$; 
however, its $O(\sqrt{T})$-regret bound is much worse in terms of $T$. 
Is it then possible to design a logarithmic-regret method whose per-round complexity is independent of $T$?\looseness=-1

\begin{table}[t]
  \caption{
    Comparisons of the regret bound and per-round/total complexity. 
    Here, $\tau_\text{solve}$ is the time for computing $\hat x_t \in \argmax_{x \in X_t}\inpr{\hat c_t, x}$, and $\tau_\text{E-proj}$/$\tau_\text{G-proj}$ is the time for the Euclidean/generalized projection; see \cref{sec:main-upper-bound,subsec:detailed-comparison}.
    Regarding the regret bound of \citet{Barmann2017-wl,Barmann2020-hh}, the dependence on~$n$ varies depending on the problem setting, which we discuss in \cref{subsec:detailed-comparison}. 
    \citet{Besbes2021-ak,Besbes2023-zm} and \citet{Gollapudi2021-ad} only claim that the total complexity is $\mathrm{poly}(n, T)$. 
    Our inspection in \cref{subsec:detailed-comparison} estimates the per-round complexity of \citet{Gollapudi2021-ad} as $O(\tau_{\text{solve}} + n^5T^3)$ or higher.\looseness=-1
    }
  \vspace{1em}
  \label{tab:comparison}
  \centering
  {
  \begin{tabular}{llll}
  \toprule
  & Regret bound & Per-round complexity & Total complexity \\
  \midrule
  \citet{Barmann2017-wl, Barmann2020-hh} & $O(\sqrt{T})$ & $O(\tau_\text{solve} + \tau_\text{E-proj} + n)$ & $\text{Per-round} \times T$ \\
  \citet{Besbes2021-ak, Besbes2023-zm} & $O(n^4\ln T)$ & Not claimed & $\mathrm{poly}(n, T)$ \\
  \citet{Gollapudi2021-ad} & $O(n\ln T)$ & Not claimed & $\mathrm{poly}(n, T)$ \\
  This work (\cref{sec:main-upper-bound}) & $O(n\ln T)$ & $O(\tau_\text{solve} + \tau_\text{G-proj} + n^2)$ & $\text{Per-round} \times T$ \\
  \bottomrule
  \end{tabular}
  }
\end{table}

\subsection{Our contributions}\label{subsec:contributions}
In this paper, we first present an $O(n \ln T)$-regret method whose per-round complexity is independent of $T$ (\cref{thm:main-upper-bound}), answering the above question affirmatively.
\Cref{tab:comparison} summarizes the comparisons of our result with prior work.
Our method is very simple: we apply the online Newton step (ONS)~\citep{Hazan2007-ta} to exp-concave loss functions that are commonly used in the \emph{universal} online learning literature (which we detail in \cref{subsec:metagrad}). 
We believe this simplicity is a strength of our method, which makes it accessible to a wider audience and easier to implement.

We then address more realistic situations where the agent's actions can be suboptimal. 
We establish a regret bound of $O(n\ln T + \sqrt{\Delta_T n\ln T})$, where~$\Delta_T$ denotes the cumulative suboptimality of the agent's actions over $T$ rounds (\cref{thm:suboptimal-feedback}). 
We also apply this result to the offline setting via the online-to-batch conversion (\cref{cor:online-to-batch}).
This bound is achieved by applying MetaGrad \Citep{van-Erven2016-mg,van-Erven2021-ji}, a universal online learning method that runs ONS with $\Theta(\ln T)$ different learning rates in parallel, to the \emph{suboptimality loss} \citep{Mohajerin-Esfahani2018-jf}, a loss function commonly used in inverse optimization.
While universal online learning is originally intended to adapt to unknown types of loss functions, our result shows that it is useful for adapting to unknown suboptimality levels in online inverse linear optimization. 
At a high level, our important contribution lies in uncovering the deeper connection between inverse optimization and online learning, thereby enabling the former to leverage the powerful toolkit of the latter.\looseness=-1

Finally, we present a regret lower bound of $\Omega(n)$ (\cref{thm:lower-bound}). 
Thus, the upper bound of $O(n \ln T)$ achieved by the method of \citet{Gollapudi2021-ad} and ours is tight up to an $O(\ln T)$ factor.
While the proof idea is somewhat straightforward, this lower bound clarifies the optimal dependence on~$n$ in the regret bound, thereby resolving a question raised in \citet[Section~7]{Besbes2023-zm}. 

\subsection{Related work}\label{subsec:related-work}
Classic studies on inverse optimization explored formulations for identifying parameters of forward optimization from a single observation \citep{Ahuja2001-cv,Iyengar2005-fl}. 
Recently, data-driven inverse optimization, which is intended to infer parameters of forward optimization from multiple noisy (possibly suboptimal) observations, has drawn significant interest \citep{Keshavarz2011-nb,Bertsimas2015-kw,Aswani2018-jf,Mohajerin-Esfahani2018-jf,Tan2020-lp,Birge2022-fq,Long2024-gl,Mishra2024-nk,Zattoni-Scroccaro2024-en}. 
This body of work has addressed offline settings with other criteria than the regret, which we formally define in~\eqref{eq:regret}.
The suboptimality loss was introduced by \citet{Mohajerin-Esfahani2018-jf} in this context.

The line of work by \citet{Barmann2017-wl,Barmann2020-hh}, \citet{Besbes2021-ak,Besbes2023-zm}, and \citet{Gollapudi2021-ad}, mentioned in \cref{sec:introduction}, is the most relevant to our work; we present detailed comparisons with them in \cref{subsec:detailed-comparison}. 
It is worth mentioning here that \citet{Gollapudi2021-ad} also obtained an $\exp(O(n\ln n))$-regret bound; 
therefore, it is possible to achieve a finite regret bound, although the dependence on $n$ is exponential.
Recently, \citet{Sakaue2025-fe} 
obtained a finite regret bound by assuming that there is a gap between the optimal and suboptimal objective values. 
Unlike their work, we do not require such gap assumptions.
Online inverse linear optimization can also be viewed as a variant of stochastic linear bandits \citep{Dani2008-uo,Abbasi-yadkori2011-st}, where noisy objective values are given as feedback, instead of optimal actions.
Intuitively, the optimal-action feedback is more informative and allows for the~$O(n \ln T)$ regret upper bound, while there is a lower bound of $\Omega(n \sqrt{T})$ in stochastic linear bandits \citep[Theorem~3]{Dani2008-uo}.
Online-learning approaches to other related settings have also been studied \citep{Jabbari2016-vr,Dong2018-ap,Ward2019-bp}; see \citet[Section~1.2]{Besbes2023-zm} for an extensive discussion on the relation to these studies.
Additionally, \citet{Chen2020-kc} and \citet{Sun2023-nc} studied online-learning methods for related settings with different criteria.

ONS \citep{Hazan2007-ta} is a well-known online convex optimization (OCO) method that achieves a logarithmic regret bound for exp-concave loss functions. 
While ONS requires the prior knowledge of the exp-concavity, universal online learning methods, including MetaGrad, can automatically adapt to the unknown curvatures of loss functions, such as the strong convexity and exp-concavity \Citep{van-Erven2016-mg,Wang2020-qv,van-Erven2021-ji,Zhang2022-ch}.
Our strategy for achieving robustness to suboptimal feedback is to combine the regret bound of MetaGrad (\cref{prop:metagrad-regret}) with the \emph{self-bounding} technique (see \cref{subsec:suboptimal-feedback} for details), which is widely adopted in the online learning literature \citep{Gaillard2014-hn,Wei2018-cq,Zimmert2021-kp}.

Contextual search \citep{Liu2021-il,Paes-Leme2022-xk} is a related problem of inferring the value of $\inpr{c^*, x_t}$ for an underlying vector $c^*$ given context vectors $x_t$. 
The method of \citet{Gollapudi2021-ad} is based on techniques developed in this context. 
Robustness to corrupted feedback is also studied in contextual search \citep{Krishnamurthy2021-un,Paes-Leme2022-ck,Paes-Leme2025-jc}. 
However, note that the problem setting is different from ours.
Also, the regret in contextual search is defined with optimal choices even under corrupted feedback, and the regret bounds scale linearly with the cumulative corruption level. 
By contrast, our regret is defined with the agent's possibly suboptimal actions and our regret bound grows only at the rate of $\sqrt{\Delta_T}$ for the cumulative suboptimality $\Delta_T$.\looseness=-1

Improving the per-round complexity is a crucial topic. 
This issue has gathered particular attention in online portfolio selection \citep{Cover1996-ib,Kalai2002-ba,Van_Erven2020-dz}.
There exists a trade-off between the per-round complexity and regret bounds among known methods for this problem, and advancing this frontier is recognized as important research \citep{Jezequel2022-dl,Zimmert2022-io,Tsai2023-oi}. 
Turning to online inverse linear optimization, logarithmic regret bounds had only been achieved by the somewhat inefficient methods of \citet{Besbes2021-ak,Besbes2023-zm} and \citet{Gollapudi2021-ad}, while the efficient online-learning approach of \citet{Barmann2017-wl,Barmann2020-hh} only enjoys the $O(\sqrt{T})$-regret bound.
This background highlights the significance of our efficient $O(n \ln T)$-regret method, which realizes the benefits of both approaches that previously existed in a trade-off relationship.\looseness=-1

\section{Preliminaries}\label{sec:preliminaries}

\subsection{Problem setting}\label{subsec:problem-setting}
We consider an online learning setting with two players, a \emph{learner} and an \emph{agent}.
The agent sequentially solves linear optimization problems of the following form for $t = 1,\dots,T$:
\begin{equation}\label{eq:decision-problem}
  \mathrm{maximize}\;\; \inpr{c^*, x} 
  \qquad 
  \mathrm{subject~to}\;\; x \in X_t,
\end{equation} 
where $c^*$ is the agent's objective vector, which is unknown to the learner. 
Every feasible set $X_t \subseteq \R^n$ is non-empty and compact, and the agent's action $x_t$ always belongs to $X_t$. 
We assume that the agent's action is optimal for~\eqref{eq:decision-problem}, i.e., $x_t \in \argmax_{x \in X_t} \inpr{c^*, x}$, except in \cref{subsec:suboptimal-feedback}, where we discuss the case where $x_t$ can be suboptimal.
The set, $X_t$, is not necessarily convex; we only assume access to an oracle that returns an optimal solution $x \in \argmax_{x' \in X_t} \inpr{c, x'}$ for any $c \in \R^n$.
If $X_t$ is a polyhedron, any solver for linear programs (LPs) of the form \eqref{eq:decision-problem} can serve as the oracle.
Even if \eqref{eq:decision-problem} is, for example, an integer LP, we may use empirically efficient solvers, such as Gurobi, to obtain an optimal solution.\looseness=-1

The learner sequentially makes a prediction of $c^*$ for $t=1,\dots,T$. 
Let $\Theta \subseteq \R^n$ denote a set of linear objective vectors, from which the learner picks predictions.
We assume that $\Theta$ is a closed convex set and that the true objective vector $c^*$ is contained in $\Theta$.
For $t=1,\dots,T$, the learner alternately outputs a prediction $\hat c_t$ of $c^*$ based on past observations $\set{(X_i, x_i)}_{i=1}^{t-1}$ and receives $(X_t, x_t)$ as feedback from the agent.
Let $\hat x_t \in \argmax_{x\in X_t}\inpr{\hat c_t, x}$ denote an optimal action induced by the learner's $t$th prediction.\footnote{We may break ties, if any, arbitrarily. Our results remain true as long as $\hat x_t$ is optimal for $\hat c_t$.}
We consider the following two measures of the quality of predictions $\hat c_1,\dots,\hat c_T \in \Theta$:\looseness=-1
\begin{equation}\label{eq:regret}
  \begin{aligned}
    R^{c^*}_T \coloneqq \sum_{t=1}^T \inpr{c^*, x_t - \hat x_t}
    &&
    \text{and}
    &&
    \Rtl^{c^*}_T 
    \coloneqq 
    R^{c^*}_T
    +
    \sum_{t=1}^T \inpr{\hat c_t, \hat x_t - x_t}
    =
    \sum_{t=1}^T \inpr{\hat c_t - c^*, \hat x_t - x_t}.  
  \end{aligned}
\end{equation}
Following prior work \citep{Besbes2021-ak,Gollapudi2021-ad,Besbes2023-zm}, we call $R^{c^*}_T$ the \emph{regret}, which is the cumulative gap between the optimal objective values and the objective values achieved by following the learner's predictions. 
Note that we have $\inpr{c^*, x_t - \hat x_t} \ge 0$ as long as $x_t$ is optimal for $c^*$. 
While the regret is a natural performance measure, the second one in \eqref{eq:regret}, $\Rtl^{c^*}_T$, is convenient when considering the online-learning approach of \citet{Barmann2017-wl,Barmann2020-hh}, as described in \cref{sec:introduction}. 
Note that $R^{c^*}_T \le \Rtl^{c^*}_T$ always holds since the additional term consisting of $\inpr{\hat c_t, \hat x_t - x_t}$ is non-negative due to the optimality of $\hat x_t$ for~$\hat c_t$; intuitively, this term quantifies how well $\hat c_t$ explains the agent's choice $x_t$.
Our upper bounds in \cref{thm:main-upper-bound,thm:suboptimal-feedback} apply to $\Rtl^{c^*}_T$, and our lower bound in \cref{thm:lower-bound} 
applies to $R^{c^*}_T$.

\begin{remark}\label[remark]{rem:setting}
  The problem setting of \citet{Besbes2021-ak,Besbes2023-zm} involves \emph{context functions} and \emph{initial knowledge sets}, which might make their setting appear more general than ours.
  However, it is not difficult to confirm that our methods are applicable to their setting.
  See \cref{subsec:detailed-comparison} for details.
\end{remark}

\subsection{Boundedness assumptions and suboptimality loss}\label{subsec:assumptions}
We introduce the following bounds on the sizes of $X_t$ and $\Theta$.
\begin{assumption}\label{assumption:bounds}
  The $\ell_2$-diameter of $\Theta$ is bounded by $D > 0$, and the $\ell_2$-diameter of $X_t$ is bounded by $K > 0$ for $t=1,\dots,T$.
  Furthermore, there exists $B > 0$ satisfying the following condition:
  \[
    \max\Set*{\inpr{c - c', x - x'}}{c, c'\in \Theta, x, x' \in X_t}  \le B 
    \quad 
    \text{for $t=1,\dots,T$.}
  \]
\end{assumption}
Assuming bounds on the diameters is common in the previous studies \citep{Barmann2017-wl,Barmann2020-hh,Besbes2021-ak,Gollapudi2021-ad,Besbes2023-zm}. 
We additionally introduce $B > 0$ to measure the sizes of $X_t$ and $\Theta$ taking their mutual relationship into account.
Note that the choice of $B = DK$ is always valid due to the Cauchy--Schwarz inequality.
This quantity is inspired by a semi-norm of gradients used in \Citet{van-Erven2021-ji} and enables sharper analysis than that conducted by simply setting $B = DK$.

We also define the \emph{suboptimality loss} for later use.
\begin{definition}\label[definition]{def:suboptimality-loss}
  For $t = 1,\dots, T$, for any action set $X_t$ and the agent's possibly suboptimal action~$x_t$, the suboptimality loss is defined by 
  $\ell_t(c) \coloneqq \max_{x\in X_t}\inpr{c, x} - \inpr{c, x_t}$ for all $c \in \Theta$.\looseness=-1  
\end{definition}
That is, $\ell_t(c)$ is the suboptimality of $x_t \in X_t$ for $c$.
\Citet{Mohajerin-Esfahani2018-jf} introduced this as a loss function that enjoys desirable computational properties in the context of inverse optimization.
Specifically, the suboptimality loss is convex, and there is a convenient expression of a subgradient.\looseness=-1  
\begin{proposition}[{cf.~\citet[Proposition~3.1]{Barmann2020-hh}}]\label[proposition]{prop:subopt-loss}
  The suboptimality loss, $\ell_t\colon\Theta\to\R$, is convex. 
  Moreover, for any $\hat c_t \in \Theta$ and $\hat x_t \in \argmax_{x \in X_t} \inpr{\hat c_t, x}$, it holds that $\hat x_t - x_t \in \partial\ell_t(\hat c_t)$. 
\end{proposition}
Confirming these properties is not difficult: 
the convexity is due to the fact that $\ell_t$ is the pointwise maximum of linear functions $c\mapsto \inpr{c, x} - \inpr{c, x_t}$, and the subgradient expression is a consequence of Danskin's theorem \citep{Danskin1966-gb} 
(or one can directly prove this as in \citet[Proposition~3.1]{Barmann2020-hh}).
It is worth mentioning that $\Rtl^{c^*}_T$ appears as the linearized upper bound on the regret with respect to the suboptimality loss, i.e., 
$\sum_{t=1}^T \prn*{\ell_t(\hat c_t) - \ell_t(c^*)} \le \sum_{t=1}^T \inpr{\hat c_t - c^*, g_t} = \Rtl^{c^*}_T$, where $g_t = \hat x_t - x_t \in \partial\ell_t(\hat c_t)$, as pointed out by \citet{Sakaue2025-fe}.  
Additionally, we have $\Rtl^{c^*}_T = R^{c^*}_T + \sum_{t=1}^T \ell_t(\hat c_t)$ in \eqref{eq:regret}.

\subsection{ONS and MetaGrad}\label{subsec:metagrad}
We briefly describe ONS and MetaGrad, based on \citet[Section~4.4]{Hazan2023-pq} and \Citet{van-Erven2021-ji}, to aid understanding of our methods. 
\Cref{asec:metagrad} shows the details for completeness.
Readers who wish to proceed directly to our results may skip this section, taking \cref{prop:ons-regret-surrogate,prop:metagrad-regret} as given.\looseness=-1

For convenience, we first state a specific form of ONS's $O(n\ln T)$ regret bound, which is later used in MetaGrad and in our analysis.
See \cref{alg:ons} in \cref{asubsec:ons} for the pseudocode of ONS.
\begin{proposition}\label[proposition]{prop:ons-regret-surrogate}
  Let $\Wcal \subseteq \R^n$ be a closed convex set whose $\ell_2$-diameter is at most $W > 0$. 
  Let $w_1,\dots,w_T$ and $g_1,\dots,g_T$ be vectors in $\R^n$ satisfying the following conditions for some $G,H > 0$:\looseness=-1 
  \begin{equation}\label{eq:gt-condition}
    \begin{aligned}
      w_t \in \Wcal, 
      &&
      \norm{g_t}_2 \le G,
      &&
      \text{and}
      &&
      \max\Set*{\inpr{w' - w, g_t}}{w, w' \in \Wcal} \le H 
      &&
      \text{for $t=1,\dots,T$}.
    \end{aligned}
  \end{equation}
  Take any $\eta \in \left(0, \frac{1}{5H}\right]$ and define loss functions $\fmg^\eta_t\colon\Wcal\to\R$ for $t=1,\dots,T$ as follows: 
  \begin{equation}\label{eq:surrogate-loss}
    \fmg^\eta_t(w) \coloneqq - \eta \inpr{w_t - w, g_t} + \eta^2 \inpr{w_t - w, g_t}^2 
    \quad 
    \text{for any $w \in \Wcal$}.
  \end{equation}
  Let $w_1^\eta,\dots, w_T^\eta \in \Wcal$ be the outputs of ONS applied to $\fmg^\eta_1,\dots,\fmg^\eta_T$. 
  Then, for any $u \in \Wcal$, it holds that 
  \[
  \sum_{t=1}^T \prn*{\fmg^\eta_t(w^\eta_t) - \fmg^\eta_t(u)}
  =
  O 
  \prn*{
    n \ln \prn*{\frac{WGT}{Hn}}
  }.
  \]
\end{proposition}

Next, we describe MetaGrad, which we apply to the following general OCO problem on a closed convex set,~$\Wcal \subseteq \R^n$. 
For $t=1,\dots,T$, we select~$w_t \in \Wcal$ based on information obtained up to the end of round~$t-1$; then, we incur $\fmg_t(w_t)$ and observe a subgradient, $g_t \in \partial \fmg_t(w_t)$, where $f_t\colon\Wcal\to\R$ denotes the $t$th convex loss function.
We assume that $\Wcal$ and $g_t$ for $t=1,\dots,T$ satisfy the conditions in \eqref{eq:gt-condition}.
Our goal is to make the regret with respect to $f_t$, i.e., $\sum_{t=1}^T \prn*{\fmg_t(w_t) - \fmg_t(u)}$, as small as possible for any comparator $u \in \Wcal$.

MetaGrad maintains \emph{$\eta$-experts}, each of whom is associated with one of $\Theta(\ln T)$ different learning rates $\eta \in \left(0, \frac{1}{5H}\right]$.
Each $\eta$-expert applies ONS to loss functions $\fmg^\eta_t$ of the form \eqref{eq:surrogate-loss}, where $w_t \in \Wcal$ is the $t$th output of MetaGrad and~$g_t \in \partial\fmg_t(w_t)$ is given as feedback. 
In each round $t$, given the outputs~$w_t^\eta$ of $\eta$-experts (which are computed based on information up to round $t-1$), MetaGrad computes $w_t \in \Wcal$ by aggregating them via the exponentially weighted average (EWA). 

For any comparator $u \in \Wcal$, define 
$\Rtl^u_T \coloneqq \sum_{t=1}^T \inpr{w_t - u, g_t}$ and $V_T^{u} \coloneqq \sum_{t=1}^T \inpr{w_t - u, g_t}^2$.
Since all functions $\fmg_t$ are convex, the regret with respect to $f_t$, or $\sum_{t=1}^T \prn*{\fmg_t(w_t) - \fmg_t(u)}$, is bounded by~$\Rtl^u_T$ from above. 
Furthermore, from the definition of $\fmg^\eta_t$, we can decompose $\Rtl^u_T$ as follows:
\begin{equation}
  \Rtl^u_T
  = 
  -\frac{\sum_{t=1}^T \fmg^\eta_t(u)}{\eta} + \eta V_T^{u}
  = 
  \frac{1}{\eta}\prn[\Bigg]{
    \sum_{t=1}^T \prn[\big]{\overbrace{\fmg^\eta_t(w_t)}^{{\text{Zero by~\eqref{eq:surrogate-loss}}}}\!- \fmg^\eta_t(w^\eta_t)}
    \!+ 
    \sum_{t=1}^T \prn*{\fmg^\eta_t(w^\eta_t) - \fmg^\eta_t(u)}} + \eta V_T^{u},
\end{equation}
which simultaneously holds for all $\eta > 0$.
The first summation on the right-hand side, i.e., the regret of EWA compared to $w^\eta_t$, is indeed as small as $O(\ln\ln T)$, while \cref{prop:ons-regret-surrogate} ensures that the second summation is $O(n\ln T)$.
Thus, the right-hand side is $O\prn*{\frac{n\ln T}{\eta} + \eta V_T^{u}}$. 
If we knew the true $V_T^{u}$ value, we could choose $\eta \simeq \sqrt{n\ln T/V_T^{u}}$ to achieve $O\prn*{\sqrt{n\ln T \cdot V_T^{u}}}$. 
This might seem impossible as we do not know $u$, and we also do not know $g_t$ or $w_t$ beforehand. 
However, we can show that at least one of $\Theta\prn*{\ln T}$ values of $\eta$ leads to almost the same regret, eschewing the need for knowing~$V^u_T$. 
Formally, MetaGrad achieves the following regret bound (cf.~\Citet[Corollary~8]{van-Erven2021-ji}).\footnote{In \Citet[Corollary~8]{van-Erven2021-ji}, the multiplicative factor of $H$ in the second term and the denominators of $Hn$ in $\ln$ are replaced with $WG$ and $n$, respectively. 
We modify it to obtain the above bound; see \cref{asec:metagrad}.
}\looseness=-1
\begin{proposition}\label[proposition]{prop:metagrad-regret}
  Let $\Wcal \subseteq \R^n$ be given as in \cref{prop:ons-regret-surrogate}.
  Let $w_1,\dots,w_T \in \Wcal$ be the outputs of MetaGrad applied to convex loss functions $\fmg_1,\dots,\fmg_T\colon\Wcal\to\R$.  
  Assume that for every $t=1,\dots,T$, subgradient $g_t \in \partial\fmg_t(w_t)$ satisfies the conditions \eqref{eq:gt-condition} in \cref{prop:ons-regret-surrogate}.
  Then, it holds that
  \[
  \sum_{t=1}^T \prn*{\fmg_t(w_t) - \fmg_t(u)}
  \le
  \Rtl^u_T
  =
  O 
  \prn*{
    \sqrt{n\ln\prn*{\frac{WGT}{Hn}} \cdot V^u_T} + H n\ln\prn*{\frac{WGT}{Hn}}
  }.
  \]
\end{proposition}

We outline how this result applies to exp-concave losses.
Taking $W$, $G$, and $H$ to be constants and ignoring the additive term of $O(n\ln(T/n))$ for simplicity, we have $\Rtl^u_T = O(\sqrt{n\ln T \cdot V^u_T})$. 
If all~$\fmg_t$ are $\alpha$-exp-concave for some $\alpha \le 1/(GW)$, then $\fmg_t(w_t) - \fmg_t(u) \le \inpr{w_t - u, g_t} - \frac{\alpha}{2}\inpr{w_t - u, g_t}^2$ holds (e.g., \citet[Lemma~4.3]{Hazan2023-pq}). 
Summing over $t$ and using \cref{prop:metagrad-regret} yield
\begin{equation}
  \sum_{t=1}^T \prn*{\fmg_t(w_t) - \fmg_t(u)} 
  \le \Rtl^u_T - \frac{\alpha}{2}V^u_T
  = O\prn*{\sqrt{n\ln T \cdot V_T^{u}} - \alpha V_T^{u}} \lesssim O\prn[\bigg]{\frac{n}{\alpha}\ln T},
\end{equation}
where the last inequality is due to $\sqrt{ax} - bx \le \frac{a}{4b}$ for any $a \ge 0$, $b > 0$, and $x \ge 0$.
Remarkably, MetaGrad achieves the $O\prn*{\frac{n}{\alpha}\ln T}$ regret bound without prior knowledge of $\alpha$, whereas ONS achieves this regret bound by using the $\alpha$ value.
Furthermore, even when some $\fmg_t$ are not exp-concave, MetaGrad still enjoys a regret bound of $O(\sqrt{T}\ln\ln T)$ \Citep[Corollary~8]{van-Erven2021-ji}.
As such, MetaGrad can automatically adapt to the unknown curvature of loss functions (at the cost of the negligible $\ln\ln T$ factor), which is the key feature of universal online learning methods.

\section{An efficient $O(n\ln T)$-regret method based on ONS}\label{sec:main-upper-bound}
This section presents an efficient logarithmic-regret method for online inverse linear optimization.
Our method is remarkably simple: we apply ONS to exp-concave loss functions defined similarly to the $\eta$-experts' losses~\eqref{eq:surrogate-loss} used in MetaGrad. 
The proof is very short given the ONS's regret bound in \cref{prop:ons-regret-surrogate}.
Despite this simplicity, we can achieve the regret bound of $O(n\ln T)$, which matches the best-known regret upper bound of \citet{Gollapudi2021-ad}, with far lower per-round complexity.

\begin{theorem}\label{thm:main-upper-bound}
  Assume that for every $t = 1,\dots,T$, action $x_t \in X_t$ is optimal for $c^* \in \Theta$.
  Let $\hat c_1,\dots, \hat c_T \in \Theta$ be the outputs of ONS applied to loss functions defined as follows 
  for $t=1,\dots,T$: 
  \begin{equation}\label{eq:eta-expert-surrogate-loss}
    \ell^\eta_t(c) \coloneqq -\eta \inpr{\hat c_t - c, \hat x_t - x_t} + \eta^2 \inpr{\hat c_t - c, \hat x_t - x_t}^2 
    \quad
    \text{for all $c \in \Theta$},
  \end{equation}
  where $\hat x_t \in \argmax_{x \in X_t}\inpr{\hat c_t, x}$ and we set $\eta = \frac{1}{5B}$.\footnote{
    This is equivalent to MetaGrad with a single $\frac{1}{5B}$-expert applied to the suboptimality losses, $\ell_1,\dots,\ell_T$.\looseness=-1 
  } 
  Then, for $R^{c^*}_T$ and $\Rtl^{c^*}_T$ in \eqref{eq:regret}, it holds that 
  \[
    R^{c^*}_T 
    \le
    \Rtl^{c^*}_T
    =
    O\prn*{Bn\ln\prn*{\frac{DKT}{Bn}}}.
  \]
\end{theorem}
\begin{proof}
  Consider using \cref{prop:ons-regret-surrogate} in the current setting with $\Wcal = \Theta$, $w^\eta_t = w_t = \hat c_t$, $g_t = \hat x_t - x_t$, $u = c^*$, $W = D$, $G = K$, and $H = B$.
  Since the optimality of $x_t$ and $\hat x_t$ for $c^*$ and $\hat c_t$, respectively, ensures $\inpr{\hat c_t - c^*, \hat x_t - x_t} \ge 0$, we have $\inpr{\hat c_t - c^*, \hat x_t - x_t}^2 \le B \inpr{\hat c_t - c^*, \hat x_t - x_t}$ due to \cref{assumption:bounds}. 
  Therefore, $\Rtl^{c^*}_T = \sum_{t=1}^T \inpr{\hat c_t - c^*, \hat x_t - x_t}$ and $V^{c^*}_T \coloneqq \sum_{t=1}^T \inpr{\hat c_t - c^*, \hat x_t - x_t}^2$ satisfy $V^{c^*}_T \le B\Rtl^{c^*}_T$. 
  By using this and \cref{prop:ons-regret-surrogate} with $\eta = \frac{1}{5B}$, for some constant $C_\mathrm{ONS} > 0$, it holds that\looseness=-1
  \begin{equation}
    \Rtl^{c^*}_T
    = 
    -
    \sum_{t=1}^T \frac{\ell^\eta_t(c^*)}{\eta}
    + 
    \eta V^{c^*}_T
    \le
    \sum_{t=1}^T 
    \frac{
      \overbrace{\ell^\eta_t(\hat c_t)}^{{\text{Zero by~\eqref{eq:eta-expert-surrogate-loss}}}}\!\!-\ \ell^\eta_t(c^*)}{\eta}
    + 
    \eta B\Rtl^{c^*}_T
    \le 
    5B
    C_\mathrm{ONS}
    n\ln \prn*{\frac{DKT}{Bn}} 
    +     
    \frac{\Rtl^{c^*}_T}{5},
  \end{equation}
  and rearranging the terms yields $\Rtl^{c^*}_T = O\prn*{Bn\ln\prn*{\frac{DKT}{Bn}}}$.\footnote{
    We may use any $\eta$ as long as $\eta B < 1$ holds; $\eta = \frac{1}{5B}$ is for consistency with MetaGrad in \cref{asec:metagrad}.
  }
  This also applies to $R^{c^*}_T \le \Rtl^{c^*}_T$.
\end{proof}

\paragraph{Time complexity.}
We discuss the time complexity of the method.
Let $\tau_\text{solve}$ be the time for solving linear optimization to find $\hat x_t$ and $\tau_\text{G-proj}$ the time for the generalized projection onto $\Theta$ used in ONS (see \cref{asubsec:ons}).
In each round $t$, we compute $\hat x_t \in \argmax_{x \in X_t}\inpr{\hat c_t, x}$ in $\tau_\text{solve}$ time; 
after that, the ONS update takes $O(n^2 + \tau_\text{G-proj})$ time. 
Therefore, it runs in $O\prn*{\tau_\text{solve} + n^2 + \tau_\text{G-proj}}$ time per round, which is independent of $T$.
If problem \eqref{eq:decision-problem} is an LP, $\tau_\text{solve}$ equals the time for solving the LP (cf.~\citet{Cohen2021-qt,Jiang2021-ih}). 
Also, $\tau_\text{G-proj}$ is often affordable as $\Theta$ is usually specified by the learner and hence has a simple structure. 
For example, if $\Theta$ is the unit Euclidean ball, the generalized projection can be computed in $O(n^3)$ time by singular value decomposition (e.g., \citet[Section~4.1]{Mhammedi2019-iw}).
We may also use the quasi-Newton-type method for further efficiency \citep{Mhammedi2023-hi}.\looseness=-1

\section{Robustness to suboptimal feedback with MetaGrad}\label{subsec:suboptimal-feedback}
In practice, assuming that the agent's actions are always optimal is unrealistic.
This section discusses how to handle suboptimal feedback effectively.
Here, we let $x_t \in X_t$ denote an arbitrary action taken by the agent, which the learner observes.
Note that $x_t$ here is possibly unrelated to $c^*$; consequently, we can no longer ensure meaningful bounds on the regret that compares $\hat x_t$ with optimal actions.
For example, if revealed actions $x_t$ remain all zeros for $t=1,\dots,T$, we can learn nothing about~$c^*$, and hence the regret that compares $\hat x_t$ with optimal actions grows linearly in~$T$ in the worst case.
Therefore, it should be noted that the regret, $R^{c^*}_T = \sum_{t=1}^T \inpr{c^*, x_t - \hat x_t}$, used here is defined with the agent's possibly suboptimal actions~$x_t$, not with those optimal for $c^*$.
Small upper bounds on this regret ensure that, if the agent's actions~$x_t$ are nearly optimal for $c^*$, so are $\hat x_t$.
Note that this regret remains to satisfy $R^{c^*}_T \le \Rtl^{c^*}_T = \sum_{t=1}^T \inpr{\hat c_t - c^*, \hat x_t - x_t}$ since~$\hat x_t$ is optimal for~$\hat c_t$.
Additionally, we note that the suboptimality loss, $\ell_t$, in \cref{def:suboptimality-loss} can be defined for any action~$x_t \in X_t$ and that $\ell_t(c^*) = \max_{x\in X_t}\inpr{c^*, x} - \inpr{c^*, x_t} \ge 0$ indicates the suboptimality of $x_t$ for~$c^*$.
Below, we use~$\Delta_T \coloneqq \sum_{t=1}^T \ell_t(c^*)$ to denote the cumulative suboptimality of the agent's actions $x_t$.

In this setting, it is not difficult to show that ONS used in \cref{thm:main-upper-bound} enjoys a regret bound that scales linearly with $\Delta_T$. 
However, the linear dependence on $\Delta_T$ is not satisfactory, as it results in a regret bound of $O(T)$ even for small suboptimality that persists across all rounds.
The following theorem ensures that by applying MetaGrad to the suboptimality losses, we can obtain a regret bound that scales with $\sqrt{\Delta_T}$.
\begin{theorem}\label{thm:suboptimal-feedback}
  Let $\hat c_1,\dots,\hat c_T \in \Theta$ be the outputs of MetaGrad applied to the suboptimality losses, $\ell_1,\dots,\ell_T$, given in \cref{def:suboptimality-loss}.
  Let $\hat x_t \in \argmax_{x \in X_t} \inpr{\hat c_t, x}$ for $t=1,\dots,T$. 
  Then, it holds that\looseness=-1
  \[
    R^{c^*}_T 
    \le
    \Rtl^{c^*}_T
    =
    O\prn*{
    Bn\ln\prn*{\frac{DKT}{Bn}}
    +
    \sqrt{
      \Delta_T Bn\ln\prn*{\frac{DKT}{Bn}}
    }
    }.
  \]
\end{theorem}
\begin{proof}
  Similar to the proof of \cref{thm:main-upper-bound}, we apply \cref{prop:metagrad-regret} with $\Wcal = \Theta$, $w_t = \hat c_t$, $g_t = \hat x_t - x_t$, $u = c^*$, $W = D$, $G = K$, and $H = B$; 
  in addition, $g_t = \hat x_t - x_t \in \partial\ell_t(\hat c_t)$ holds due to \cref{prop:subopt-loss}. 
  Thus, \cref{prop:metagrad-regret} ensures the following bound for some constant $C_\mathrm{MG} > 0$: 
  \begin{equation}\label{eq:suboptimal-metagrad-regret}
    \Rtl^{c^*}_T
    \le 
    C_\mathrm{MG} \prn*{
      \sqrt{n\ln\prn*{\frac{DKT}{Bn}} \cdot V^{c^*}_T } + Bn\ln\prn*{\frac{DKT}{Bn}}
    },
  \end{equation}  
  where $\Rtl^{c^*}_T = \sum_{t=1}^T \inpr{\hat c_t - c^*, \hat x_t - x_t}$ and $V^{c^*}_T = \sum_{t=1}^T \inpr{\hat c_t - c^*, \hat x_t - x_t}^2$.
  In contrast to the case of \cref{thm:main-upper-bound}, $\inpr{\hat c_t - c^*, \hat x_t - x_t}^2 \le B\inpr{\hat c_t - c^*, \hat x_t - x_t}$ is not ensured since $\inpr{\hat c_t - c^*, \hat x_t - x_t}$ can be negative due to the suboptimality of $x_t$.
  Instead, we will show that the following inequality holds:
  \begin{equation}\label{eq:subopt-instantaneous}
    \inpr{\hat c_t - c^*, \hat x_t - x_t}^2 \le B\inpr{\hat c_t - c^*, \hat x_t - x_t} + 2B\ell_t(c^*).
  \end{equation}
  If $\inpr{\hat c_t - c^*, \hat x_t - x_t} \ge 0$, \eqref{eq:subopt-instantaneous} is immediate from
  $\inpr{\hat c_t - c^*, \hat x_t - x_t}^2 \le B\inpr{\hat c_t - c^*, \hat x_t - x_t}$ and $\ell_t(c^*) \ge 0$. 
  If $\inpr{\hat c_t - c^*, \hat x_t - x_t} < 0$, $\inpr{\hat c_t - c^*, \hat x_t - x_t}^2 \le B\prn*{-\inpr{\hat c_t - c^*, \hat x_t - x_t}}$ holds. 
  In addition, we have\looseness=-1
  \[
    \ell_t(c^*) = \max_{x\in X_t}\inpr{c^*, x} - \inpr{c^*, x_t}  
    \ge \inpr{c^*, \hat x_t - x_t} \ge \inpr{c^*, \hat x_t - x_t} - \inpr{\hat c_t, \hat x_t - x_t}
    =
    - \inpr{\hat c_t - c^*, \hat x_t - x_t}    
    ,
  \]
  where the second inequality follows from $\inpr{\hat c_t, \hat x_t - x_t} \ge 0$.
  Multiplying both sides by $2$ yields
  \[
    -2\inpr{\hat c_t - c^*, \hat x_t - x_t} \le 2\ell_t(c^*)
    \iff
    -\inpr{\hat c_t - c^*, \hat x_t - x_t} \le \inpr{\hat c_t - c^*, \hat x_t - x_t} + 2\ell_t(c^*).
  \]
  Thus, \eqref{eq:subopt-instantaneous} holds in any case, and hence $V^{c^*}_T \le B \Rtl^{c^*}_T + 2B\Delta_T$. 
  Substituting this into \eqref{eq:suboptimal-metagrad-regret}, we obtain
  \[
    \Rtl^{c^*}_T
    \le 
    C_\mathrm{MG} \prn*{
      \sqrt{Bn\ln\prn*{\frac{DKT}{Bn}}
        \prn*{
          \Rtl^{c^*}_T
          +
          2\Delta_T
          }}
      + Bn\ln\prn*{\frac{DKT}{Bn}}
    }.
  \]
  We assume $\Rtl^{c^*}_T> 0$; otherwise, the trivial bound of $\Rtl^{c^*}_T \le 0$ holds. 
  By the subadditivity of $x\mapsto\sqrt{x}$ for $x \ge 0$, we have 
  $\Rtl^{c^*}_T \le \sqrt{
      \text{$a\Rtl^{c^*}_T$}
  } + b$, where $a = C_\mathrm{MG}^2Bn\ln\prn*{\frac{DKT}{Bn}}$ and $b = \sqrt{2a\Delta_T} + \frac{a}{C_\mathrm{MG}}$. 
  Since $x \le \sqrt{ax} + b$ implies 
  $x 
  =
  \frac43 x - \frac{x}{3}  
  \le 
  \frac43(\sqrt{ax} + b) - \frac{x}{3}
  =
  -\frac13(\sqrt{x} - 2\sqrt{a})^2 + \frac43(a + b)
  \le  
  \frac43(a + b)$ for any $a,b,x \ge 0$,
  we obtain $\Rtl^{c^*}_T \le \frac43(a + b) = O\prn[\Big]{
    Bn\ln\prn*{\frac{DKT}{Bn}}
    +
    \sqrt{
      \Delta_T Bn\ln\prn*{\frac{DKT}{Bn}}
    }
  }$.
\end{proof}

If every $x_t$ is optimal, i.e., $\Delta_T = 0$, the bound recovers that in \cref{thm:main-upper-bound}.
Note that MetaGrad requires no prior knowledge of $\Delta_T$; 
it automatically achieves the bound that scales with $\sqrt{\Delta_T}$, analogous to the original bound in \cref{prop:metagrad-regret} that scales with $\sqrt{V^u_T}$. 
Moreover, a refined version of MetaGrad \Citep{van-Erven2021-ji} enables us to achieve a similar bound without prior knowledge of $K$, $B$, or~$T$ (see \cref{asubsec:lipschitz-adaptivity}).
Universal online learning methods shine in such scenarios where adaptivity to unknown quantities is desired.
Another noteworthy point is that the last part of the proof uses the self-bounding technique \citep{Gaillard2014-hn,Wei2018-cq,Zimmert2021-kp}. 
Specifically, we derived $\Rtl^{c^*}_T\!\lesssim\!a + b$ from $\Rtl^{c^*}_T\!\le\!\sqrt{
    \text{$a\Rtl^{c^*}_T$}
} + b$, where the latter means that $\Rtl^{c^*}_T$ is upper bounded by a term of lower order in $\Rtl^{c^*}_T$ itself, hence the name self-bounding.
We expect that the combination of universal online learning methods and self-bounding, through relations like $V^{c^*}_T\!\lesssim \Rtl^{c^*}_T + \Delta_T$ used above, will be a useful technique for deriving meaningful guarantees in online inverse linear optimization.

\paragraph{Time complexity.}
The use of MetaGrad comes with a slight increase in time complexity. 
First, as with the case of ONS, $\hat x_t \in \argmax_{x \in X_t}\inpr{\hat c_t, x}$ is computed in each round, taking $\tau_\text{solve}$ time. 
Then, each $\eta$-expert performs the ONS update, taking $O(n^2 + \tau_\text{G-proj})$ time.
Since MetaGrad maintains~$\Theta\prn*{\ln T}$ distinct $\eta$ values, the total per-round time complexity is $O\prn*{\tau_\text{solve} + (n^2 + \tau_\text{G-proj})\ln T}$.
If the $O(\tau_\text{G-proj}\ln T)$ factor is a bottleneck, we can use more efficient universal algorithms \citep{Mhammedi2019-iw,Yang2024-is} to reduce the number of projections from $\Theta(\ln T)$ to $1$. 
Moreover, the~$O(n^2)$ factor can also be reduced by sketching techniques (see \Citet[Section~5]{van-Erven2021-ji}).\looseness=-1 


\subsection{Online-to-batch conversion}\label{subsec:online-to-batch}
We briefly discuss the implication of \cref{thm:suboptimal-feedback} in the offline setting, where feedback follows some underlying distribution. 
As noted in \cref{subsec:assumptions}, the bound in \cref{thm:suboptimal-feedback} applies to the regret with respect to the suboptimality loss, $\sum_{t=1}^T\prn*{\ell_t(\hat c_t) - \ell_t(c^*)}$, since it is bounded by $\Rtl^{c^*}_T$ from above.
Therefore, the standard online-to-batch conversion (e.g., \citet[Theorem~3.1]{orabona2023modern}) implies the following convergence of the average prediction in terms of the suboptimality loss.
\begin{corollary}\label[corollary]{cor:online-to-batch}
  For any non-empty and compact $X \subseteq \R^n$, $x \in X$, and $c \in \Theta$, define the corresponding suboptimality loss as $\ell_{X, x}(c) \coloneqq \max_{x' \in X} \inpr{c, x'} - \inpr{c, x}$. 
  Let $\Delta > 0$ and define $\mathcal{X}_\Delta$ as the set of observations $(X, x)$ with bounded suboptimality, $\ell_{X, x}(c^*) \le \Delta$.
  Assume that $\set{(X_t, x_t)}_{t=1}^T$ are drawn i.i.d.\ from some distribution on $\mathcal{X}_\Delta$ (hence $\Delta_T \le \Delta T$). 
  Let $\hat c_1,\dots,\hat c_T \in \Theta$ be the outputs of MetaGrad applied to the suboptimality losses $\ell_t = \ell_{X_t, x_t}$ for $t=1,\dots,T$.
  Then, it holds that 
  \[
    \E\brc*{
      \ell_{X, x}\prn*{
        \frac{1}{T}\sum_{t=1}^T \hat c_t
      }
      - 
      \ell_{X, x}\prn*{
        c^*
      } 
    }
    = 
    O\prn[\Bigg]{
      \frac{Bn}{T}\ln\prn*{\frac{DKT}{Bn}}
      +
      \sqrt{
      \frac{\Delta Bn}{T}\ln\prn*{\frac{DKT}{Bn}}
      }
    }. 
  \] 
\end{corollary}

\citet[Theorem~3.14]{Barmann2020-hh} also obtained a similar offline guarantee via the online-to-batch conversion.
Their convergence rate is $O\prn[\big]{\frac{1}{\sqrt{T}}}$ even when $\Delta = 0$, whereas our \cref{cor:online-to-batch} offers the faster rate of $O\prn[\big]{\frac{\ln T}{T}}$ if $\Delta = 0$. 
It also applies to the case of $\Delta > 0$, which is important in practice because stochastic feedback is rarely optimal at all times. 
We emphasize that if regret bounds scale linearly with $\Delta_T$, the above online-to-batch conversion cannot ensure that the excess suboptimality loss (the left-hand side) converge to zero as $T\to0$.
This observation lends support to the importance of the $\sqrt{\Delta_T}$-dependent regret bound we established in \cref{thm:suboptimal-feedback}.\looseness=-1 

\section{$\Omega(n)$ lower bound}\label{sec:lower-bound}
We construct an instance where any online learner incurs an $\Omega(n)$ regret, implying that the $O(n \ln T)$ upper bound is tight up to an $O(\ln T)$ factor.
More strongly, the following \cref{thm:lower-bound} shows that, for any $B > 0$ that gives the tight upper bound in \cref{assumption:bounds},
no learner can achieve a regret smaller than $\frac{Bn}{4}$, which means that the $Bn$ factor in our \cref{thm:main-upper-bound} is inevitable.
\begin{theorem}\label{thm:lower-bound}
  Let $n$ be a positive integer and $\Theta = \brc*{-\frac{1}{\sqrt{n}}, +\frac{1}{\sqrt{n}}}^n$.
  For any $T \ge n$, $B > 0$, and the learner's outputs $\hat c_1,\dots,\hat c_T \in \Theta$, there exist $c^* \in \Theta$ and $X_1,\dots,X_T \subseteq \R^n$ such that\looseness=-1 
  \begin{equation}
    \begin{aligned}
      \max_{t = 1,\dots,T} \max \Set*{
        \inpr{c - c', x - x'}
      }{
        {c,c'\in\Theta, x,x' \in X_t}
      }
      = B
      &&
      \text{and}
      &&
      \E\brc*{R^{c^*}_T} \ge \frac{Bn}{4} 
    \end{aligned}
  \end{equation} 
  hold, where 
  $R^{c^*}_T = \sum_{t=1}^T \inpr{c^*, x_t - \hat x_t}$, $x_t \in \argmax_{x \in X_t} \inpr{c^*, x}$, $\hat x_t \in \argmax_{x \in X_t} \inpr{\hat c_t, x}$, and the expectation is taken over the learner's possible randomness.
\end{theorem}
\begin{proof}
  We focus on the first $n$ rounds and show that any learner must incur $\frac{Bn}{4}$ in these rounds; 
  in the remaining rounds, we may use any instance since the optimality of $x_t$ for $c^*$ ensures $\inpr{c^*, x_t - \hat x_t} \ge 0$.
  For $t = 1,\dots,n$, let $X_t = \Set*{x \in \R^n}{-\frac{B}{4}\sqrt{n} \le x(t) \le \frac{B}{4}\sqrt{n},\ x(i) = 0 \text{ for $i \neq t$}}$, where~$x(i)$ denotes the $i$th element of $x$. 
  That is, $X_t$ is the line segment on the $t$th axis from $-\frac{B}{4}\sqrt{n}$ to $\frac{B}{4}\sqrt{n}$. 
  Then, 
  $\max\Set*{
    \inpr{c - c', x - x'}
  }{
    {c,c'\in\Theta, x,x' \in X_t}
  } = B$ 
  holds for each $t = 1,\dots,n$. 
  Let $c^* \in \Theta$ be a random vector such that each entry is $-\frac{1}{\sqrt{n}}$ or $\frac{1}{\sqrt{n}}$ with probability $\frac12$, which is drawn independently of any other randomness. 
  Then, the optimal action, $x_t \in X_t$, which is zero everywhere except that its $t$th coordinate equals $\frac{c^*(t)}{|c^*(t)|}\cdot\frac{B}{4}\sqrt{n}$, achieves $\inpr{c^*, x_t} = \frac{B}{4}$. 
  Note that the learner's $t$th prediction~$\hat c_t$ is independent of $c^*(t)$ since it depends only on past observations, $\set{(X_i, x_i)}_{i=1}^{t-1}$, which have no information about $c^*(t)$. 
  Thus, $\hat x_t \in \argmax_{x \in X_t} \inpr{\hat c_t, x}$ is also independent of $c^*(t)$, and hence\looseness=-1   
  \[
    \E\brc*{\inpr{c^*, x_t - \hat x_t}} 
    = 
    \E\brc*{\inpr{c^*, x_t}}
    - 
    \E\brc*{\inpr{c^*, \hat x_t}}
    =    
    \frac{B}{4} - \frac12\prn*{-\frac{1}{\sqrt{n}} + \frac{1}{\sqrt{n}}}\hat x_t(t)
    =
    \frac{B}{4},
  \]
  where the expectation is taken over the randomness of $c^*$. 
  This implies that any deterministic learner incurs $\frac{Bn}{4}$ in the first $n$ rounds in expectation.
  Thanks to Yao's minimax principle \citep{Yao1977-ev}, we can conclude that for any randomized learner, there exists $c^* \in \Theta$ such that $\E\brc*{R^{c^*}_T} \ge \frac{Bn}{4}$ holds. 
\end{proof}
In the above proof, we restricted $X_1,\dots,X_T$ to line segments so that each $x_t \in \argmax_{x \in X_t}\inpr{c^*, x}$ reveals nothing about $c^*(t+1),\dots,c^*(n)$. 
Whether a similar lower bound holds when all $X_t$ are full-dimensional remains an open question.
Another side note is that the $\Omega(n)$ lower bound does not contradict the $O(\sqrt{T})$ upper bound of \citet{Barmann2020-hh}.
Their OGD-based method indeed achieves a regret bound of $O(DK\sqrt{T})$, where $D$ and $K$ are upper bounds on the $\ell_2$-diameters of $\Theta$ and $X_t$, respectively.
In the above proof, $T \ge n$, $D \ge 1$, and $K \ge \frac{B}{2}\sqrt{n}$ hold, implying that their regret upper bound is $DK\sqrt{T} \gtrsim Bn$.
Hence, the $\Omega(n)$-lower bound and their $O(DK\sqrt{T})$-upper bound are compatible.\looseness=-1

\section{Conclusion and discussion}\label{sec:conclusion}
We have presented an efficient ONS-based method that achieves an $O(n \ln T)$-regret bound for online inverse linear optimization. 
Then, we have extended the method to deal with suboptimal feedback based on MetaGrad, achieving an $O(n \ln T + \sqrt{\Delta_T n \ln T})$-regret bound, where $\Delta_T$ is the cumulative suboptimality of the agent's actions. 
Finally, we have presented a lower bound of $\Omega(n)$, which shows that the $O(n \ln T)$ upper bound is tight up to an $O(\ln T)$ factor.
Regarding limitations, our work is restricted to the case where the agent's optimization problem is linear, as mentioned in \cref{footnote:linear-model}; how to deal with non-linearity is an important direction for future work. 
In online portfolio selection, ONS is efficient but inferior to the universal portfolio algorithm regarding the dependence on the gradient norm \citep{Van_Erven2020-dz}. 
Exploring possible similar relationships in online inverse linear optimization is left for future work.
Last but not least, closing the $O(\ln T)$ gap between the upper and lower bounds is an important open problem. 
Interestingly, if all $X_t$ are line segments as in \cref{sec:lower-bound} and the learner can observe $X_t$ in the beginning of round $t$, the algorithm of \citet[Theorem~5.2]{Gollapudi2021-ad} offers a regret upper bound of $O(n^5 \log^2 n)$, which is finite and polynomial in $n$.
We also provide an additional discussion on a finite regret bound for the case of $n = 2$ in \cref{asec:removing-lnT}.
\looseness=-1

\subsubsection*{Acknowledgements}
SS was supported by JST ERATO Grant Number JPMJER1903.  
TT is supported by JST ACT-X Grant Number JPMJAX210E and JSPS KAKENHI Grant Number JP24K23852.
HB is supported by JST PRESTO Grant Number JPMJPR24K6.
TO is supported by JST ERATO Grant Number JPMJER1903, JST CREST Grant Number JPMJCR24Q2, JST FOREST Grant Number JPMJFR232L, JSPS KAKENHI Grant Numbers JP22K17853 and JP24K21315, and Start-up Research Funds in ICReDD, Hokkaido University.

\printbibliography[heading=bibintoc]

\clearpage
\appendix

\section{Detailed comparisons with previous results}\label[appendix]{subsec:detailed-comparison}
Below we compare our results with \citet{Barmann2017-wl,Barmann2020-hh}, \citet{Besbes2021-ak,Besbes2023-zm}, and \citet{Gollapudi2021-ad}.

\Citet{Barmann2017-wl,Barmann2020-hh} used $\Rtl^{c^*}_T$ as the performance measure, as with our \cref{thm:main-upper-bound,thm:suboptimal-feedback}, and provided two specific methods. 
The first one, based on the multiplicative weights update (MWU), is tailored for the case where $\Theta$ is the probability simplex, i.e., $\Theta = \set*{c \in \R^n \mid c \ge 0, \norm{c}_1 = 1}$. 
The authors assumed a bound of $K_\infty > 0$ on the $\ell_\infty$-diameters of $X_t$ and obtained a regret bound of $O(K_\infty\sqrt{T\ln n})$. 
The second one is based on the online gradient descent (OGD) and applies to general convex sets $\Theta$. 
The authors assumed that the $\ell_2$-diameters of $\Theta$ and $X_t$ are bounded by $D > 0$ and $K > 0$, respectively, and obtained a regret bound of $O(DK\sqrt{T})$.
In the first case, our \cref{thm:main-upper-bound} with $B = K_\infty$, $D = \sqrt{2}$, and $K \le 2\sqrt{n}K_\infty$ offers a bound of $O(K_\infty n\ln(T/\sqrt{n}))$; 
in the second case, we obtain a bound of $O(DKn\ln(T/n))$ by setting $B = DK$. 
In both cases, our bounds improve the dependence on $T$ from $\sqrt{T}$ to $\ln T$, while scaled up by a factor of $n$, up to logarithmic terms. 
Regarding the computation time, their MWU and OGD methods run in $O\prn*{\tau_\text{solve} + \tau_\text{E-proj} + n}$ time per round, where~$\tau_\text{E-proj}$ is the time for the Euclidean projection onto $\Theta$, hence faster than our method.
Also, suboptimal feedback is discussed in \citet[Sections~3.1]{Barmann2020-hh}.
However, their bound does not achieve the logarithmic dependence on $T$ even when $\Delta_T = 0$, unlike our \cref{thm:suboptimal-feedback}.\looseness=-1

\Citet{Besbes2021-ak,Besbes2023-zm} used $R^{c^*}_T$ as the performance measure, which is upper bounded by $\Rtl^{c^*}_T$. 
They assumed that $c^*$ lies in the unit Euclidean sphere and that the $\ell_2$-diameters of $X_t$ are at most one. 
Under these conditions, they obtained the first logarithmic regret bound of $O(n^4\ln T)$.
By applying \cref{thm:main-upper-bound} to this case, we obtain a bound of $O(n\ln(T/n))$, which is better than their bound by a factor of $n^3$.
As discussed in \cref{sec:introduction}, their method relies on the idea of narrowing down regions represented with $O(T)$ constraints, and hence it seems inefficient for large $T$; indeed \citet[Theorem~4]{Besbes2023-zm} only claims that the total time complexity is polynomial in $n$ and $T$.
Considering this, our ONS-based method is arguably much faster while achieving the better regret bound.\looseness=-1

\textbf{On the problem setting of \citet{Besbes2021-ak,Besbes2023-zm}.\;\;}
As mentioned in \cref{rem:setting}, the problem setting of \citet{Besbes2021-ak,Besbes2023-zm} is seemingly different from ours.
In their setting, in each round $t$, the learner first observes $(X_t, f_t)$, where $f_t\colon X_t\to\R^n$ is called a \emph{context function}.
Then, the learner chooses $\hat x_t \in X_t$ and receives an optimal action $x_t \in \argmax_{x \in X_t} \inpr{c^*, f_t(x)}$ as feedback.
It is assumed that the learner can solve $\max_{x \in X_t} \inpr{c, f_t(x)}$ for any $c \in \R^n$ and that all $f_t$ are $1$-Lipschitz, i.e., $\norm{f_t(x) - f_t(x')}_2 \le \norm{x - x'}_2$ for all $x, x' \in X_t$.
We note that our methods work in this setting, while the presence of $f_t$ might make their setting appear more general.
Specifically, we redefine $X_t$ as the image of $f_t$, i.e., $\Set*{f_t(x)}{x \in X_t}$. 
Then, their assumption ensures that we can find $f_t(\hat x_t) \in X_t$ that maximizes $X_t \ni \xi \mapsto \inpr{\hat c_t, \xi}$, and the $\ell_2$-diameter of the newly defined $X_t$ is bounded by $1$ due to the $1$-Lipschitzness of $f_t$.
Therefore, by defining $g_t = f_t(\hat x_t) - f_t(x_t)$ and applying it in \cref{thm:main-upper-bound,thm:suboptimal-feedback}, we recover the bounds therein on $\sum_{t=1}^T \inpr{\hat c_t - c^*, f_t(\hat x_t) - f_t(x_t)}$, with $D$, $K$, and $B$ being constants. 
The bounds also apply to the regret, $\sum_{t=1}^T \inpr{c^*, f_t(x_t) - f_t(\hat x_t)}$, used in \citet{Besbes2021-ak,Besbes2023-zm}.
Additionally, \citet{Besbes2021-ak,Besbes2023-zm} consider a (possibly non-convex) initial knowledge set $C_0 \subseteq \R^n$ that contains~$c^*$.
We note, however, that they do not care about whether predictions~$\hat c_t$ lie in $C_0$ or not since the regret, their performance measure, does not explicitly involve $\hat c_t$.
Indeed, predictions $\hat c_t$ that appear in their method are chosen from ellipsoidal cones that properly contain $C_0$ in general.
Therefore, our methods carried out on a convex set $\Theta \supseteq C_0$ work similarly in their setting.\looseness=-1

\citet{Gollapudi2021-ad} studied essentially the same problem as online inverse linear optimization under the name of contextual recommendation (where they and \citet{Besbes2021-ak,Besbes2023-zm} appear to have been unaware of each other's work).
As with \citet{Besbes2021-ak,Besbes2023-zm}, 
\citet{Gollapudi2021-ad} assumed that $c^*$ and $X_1,\dots,X_T$ lie in the unit Euclidean ball, denoted by $\mathbb{B}^n$. 
Similar to \citet{Besbes2021-ak,Besbes2023-zm}, their method maintains the region $K_t$, which is the intersection of hyperplanes $\Set*{c \in \R^d}{ \inpr{c - \hat c_s, x_s - \hat x_s} \ge 0}$ for $s = 1,\dots,t-1$, and sets $\hat c_t$ to the centroid of $K_t + \frac1T \mathbb{B}^n$, where $+$ is the Minkowski sum. 
As regards the regret analysis, their key idea is to use the approximate Gr\"unbaum theorem: 
whenever the learner incurs $\inpr{c^*, x_t - \hat x_t} \ge \frac{1}{T}$, $\mathsf{Vol}\prn*{K_t + \frac{1}{T}\mathbb{B}^n}$ decreases by a constant factor, where $\mathsf{Vol}$ denotes the volume. 
Consequently, $\mathsf{Vol}\prn*{K_1 + \frac{1}{T}\mathbb{B}^n} / \mathsf{Vol}\prn*{K_T + \frac{1}{T}\mathbb{B}^n} \lesssim T^n$ implies the regret bound of $R_T^{c^*} = \sum_t \inpr{c^*, x_t - \hat x_t} = O(n \ln T)$. 
As such, the per-round complexity of their method also inherently depends on $T$, and \citet[Section~1.2]{Gollapudi2021-ad} only claims the total time complexity of $\mathrm{poly}(n, T)$.
In this setting, our ONS-based method achieves a regret bound of $O(n\ln(T/n))$ and is arguably more efficient since the per-round complexity is independent of $T$.

\paragraph{Estimating the per-round complexity of \citet{Gollapudi2021-ad}.}
As described above, the method of \citet{Gollapudi2021-ad} requires $\hat x_t$ for each $t$, and hence the per-round complexity involves $\tau_\text{solve}$, the time to solve $\max_{x \in X_t} \inpr{\hat c_t, x}$.
Aside from this, its per-round complexity is dominated by the cost for computing the centroid of $K_t + \frac{1}{T}\mathbb{B}^n$, where $K_t$ is represented by $O(T)$ hyperplanes.
It is known that the problem of exactly computing the centroid is \#P-hard in general, but we can approximate it via sampling with a membership oracle of $K_t + \frac{1}{T}\mathbb{B}^n$.
To the best of our knowledge, computing a point that is $\varepsilon$-close to the centroid takes $O(n^4/\varepsilon^2)$ membership queries, up to logarithmic factors \citep[Theorem~5.7]{Feldman2015-cp}, and it is natural to set $\varepsilon=1/T$ to make the approximation error negligible. 
Thus, it takes $O(n^4T^2)$ membership queries. 
Regarding the complexity of the membership oracle, naively checking whether a given point satisfies all the $O(T)$ linear constraints of $K_t$ takes $O(nT)$ time. 
Handling the Minkowski sum with $\frac1T\mathbb{B}$ would complicates the procedure, though it can be done in $\mathrm{poly}(n, T)$ time by using, for example, Frank--Wolfe-type algorithms \citep{Frank1956-yc,Jaggi2013-ts,Lacoste-Julien2015-ir,Garber2021-tg}. 
For now, $O(nT)$ would be a reasonable (optimistic) estimate of the complexity of the membership oracle.
Consequently, the total per-round complexity of their method is estimated to be $O\prn*{\tau_\text{solve} + n^5T^3}$ (or higher).\looseness=-1

\section{Details of ONS and MetaGrad}\label[appendix]{asec:metagrad}
We present the details of ONS and MetaGrad. 
The main purpose of this section is to provide simple descriptions and analyses of those algorithms, thereby assisting readers who are not familiar with them. 
As in \cref{asubsec:lipschitz-adaptivity}, we can also derive a regret bound of MetaGrad that yields a similar result to \cref{thm:suboptimal-feedback} directly from the results of \Citet{van-Erven2021-ji}.

First, we discuss the regret bound of ONS used by $\eta$-experts in MetaGrad, proving \cref{prop:ons-regret-surrogate}.
Then, we establish the regret bound of MetaGrad in \cref{prop:metagrad-regret}.\looseness=-1

\begin{algorithm}[tb]
	\caption{Online Newton Step}
	\label{alg:ons}
	\begin{algorithmic}[1]
    \State Set $\gamma = \frac12\min\set[\big]{\frac{1}{\beta}, \alpha}$, $\varepsilon = \frac{n}{W^2\gamma^2}$, $A_0 = \varepsilon I_n$, and $w_1 \in \Wcal$.
    \For{$t=1,\dots,T$}
    \State Play $w_t$ and observe $\qons_t$.
    \State $A_t \gets A_{t-1} + \nabla\qons_t(w_t) \nabla\qons_t(w_t)^\top$.
    \State $w_{t+1} \gets \argmin \Set*{\norm*{w_t - \frac{1}{\gamma}A_t^{-1}\nabla\qons_t(w_t) - w}_{A_t}^2}{w \in \Wcal}$.
    \hfill{$\triangleright$ Generalized projection.}
    \EndFor
	\end{algorithmic}
\end{algorithm}

\subsection{Regret bound of ONS}\label[appendix]{asubsec:ons}
Let $I_n \in$ be the $n\times n$ identity matrix.
For any $A, B \in \R^{n\times n}$, $A \succeq B$ means that $A - B$ is positive semidefinite.
For positive semidefinite $A \in \R^{n\times n}$, let $\norm{x}_{A} = \sqrt{x^\top Ax}$ for $x \in \R^n$. 
Let $\Wcal \subseteq \R^n$ be a closed convex set. 
A function $\qons:\Wcal\to\R$ is \emph{$\alpha$-exp-concave} for some $\alpha > 0$ if $\Wcal\ni w \mapsto \mathrm{e}^{-\alpha \qons(w)}$ is concave.
For twice differentiable~$\qons$, this is equivalent to $\nabla^2 q(w) \succeq \alpha \nabla \qons(w) \nabla \qons(w)^\top$.
The following regret bound of ONS mostly comes from the standard analysis \citep[Section~4.4]{Hazan2023-pq}, and hence readers familiar with it can skip the subsequent proof.
The only modification lies in the use of $\beta$ instead of $W\lambda$ (defined below), where $\beta \le W\lambda$ always holds and hence slightly tighter. 
This leads to the multiplicative factor of~$B$, rather than $DK$, in \cref{thm:main-upper-bound,thm:suboptimal-feedback}.\looseness=-1

\begin{proposition}\label[proposition]{prop:ons-regret}
  Let $\Wcal \subseteq \R^n$ be a closed convex set with the $\ell_2$-diameter of at most $W > 0$.
  Assume that $\qons_1,\dots,\qons_T\colon\Wcal\to\R$ are twice differentiable and $\alpha$-exp-concave for some $\alpha > 0$. 
  Additionally, assume that there exist $\beta, \lambda > 0$ such that 
  $\max_{w\in\Wcal}\abs*{\nabla\qons_t(w_t)^\top(w - w_t)} \le \beta$ and $\norm{\nabla\qons_t(w_t)}_2 \le \lambda$ hold.
  Let $w_1,\dots,w_T \in \Wcal$ be the outputs of ONS (\cref{alg:ons}).
  Then, for any $u \in \Wcal$, it holds that\looseness=-1
  \[
  \sum_{t=1}^T \prn*{\qons_t(w_t) - \qons_t(u)}
  \le 
  \frac{n}{2\gamma}
      \prn*{
        \ln \prn*{\frac{W^2\gamma^2\lambda^2T}{n} + 1} + 1
      },
  \]
  where $\gamma = \frac12\min\set[\big]{\frac{1}{\beta}, \alpha}$ is the parameter used in ONS.
\end{proposition}

\begin{proof}
  We first give a useful inequality that follows from the $\alpha$-exp-concavity.
  By the same analysis as the proof of \citet[Lemma~4.3]{Hazan2023-pq}, for $\gamma \le \frac{\alpha}{2}$, we have
  \[
    \qons_t(w_t) - \qons_t(u) 
    \le 
    \frac{1}{2\gamma}\ln\prn*{
      1 - 2\gamma\nabla\qons_t(w_t)^\top(u - w_t)
    }.
  \]
  Note that we also have $\abs*{2\gamma\nabla\qons_t(w_t)^\top(u - w_t)} \le 2\gamma\beta \le 1$. 
  Since $\ln(1 - x) \le -x - x^2/4$ holds for $x \ge -1$, applying this with $x = 2\gamma\nabla\qons_t(w_t)^\top(u - w_t)$ yields
  \begin{equation}\label{eq:regret-quadratic}
    \qons_t(w_t) - \qons_t(u) 
    \le 
    \nabla\qons_t(w_t)^\top(w_t - u) - \frac{\gamma}{2}(w_t - u)^\top\nabla\qons_t(w_t)\nabla\qons_t(w_t)^\top(w_t - u).
  \end{equation}

  We turn to the iterates of ONS.
  Since $w_{t+1}$ is the projection of $w_t - \frac{1}{\gamma}A_t^{-1}\nabla\qons_t(w_t)$ onto $\Wcal$ with respect to the norm $\norm{\cdot}_{A_t}$, we have 
  $\norm{w_{t+1} - u}_{A_t}^2
  \le
  \norm*{w_t - \frac{1}{\gamma}A_t^{-1}\nabla\qons_t(w_t) - u}_{A_t}^2$ for $u \in \Wcal$ due to the Pythagorean theorem, hence
  \begin{align}
    &(w_{t+1} - u)^\top A_t(w_{t+1} - u) 
    \\
    \le{}& \prn*{w_t - \frac{1}{\gamma}A_t^{-1}\nabla\qons_t(w_t) - u}^\top A_t\prn*{w_t - \frac{1}{\gamma}A_t^{-1}\nabla\qons_t(w_t) - u}
    \\
    ={}& \prn*{w_t - u}^\top A_t\prn*{w_t - u} - \frac{2}{\gamma}\nabla\qons_t(w_t)^\top\prn*{w_t - u} + \frac{1}{\gamma^2}\nabla\qons_t(w_t)^\top A_t^{-1}\nabla\qons_t(w_t).
  \end{align}
  Rearranging the terms, we obtain 
  \begin{align}
    &\nabla\qons_t(w_t)^\top\prn*{w_t - u} 
    \\
    \le{}& \frac{1}{2\gamma}\nabla\qons_t(w_t)^\top A_t^{-1}\nabla\qons_t(w_t) + \frac{\gamma}{2}(w_{t} - u)^\top A_t(w_{t} - u) - \frac{\gamma}{2}\prn*{w_{t+1} - u}^\top A_t\prn*{w_{t+1} - u}.
  \end{align}
  From $A_{t} = A_{t-1} + \nabla\qons_t(w_t)\nabla\qons_t(w_t)^\top$, summing over $t$ and ignoring $\frac{\gamma}{2}\prn*{w_{T+1} - u}^\top A_T\prn*{w_{T+1} - u}$ $\ge 0$, we obtain
  \begin{align}
    &\sum_{t=1}^T \nabla\qons_t(w_t)^\top\prn*{w_t - u} 
    \\
    \le{}&
    \frac{1}{2\gamma}\sum_{t=1}^T \nabla\qons_t(w_t)^\top A_t^{-1}\nabla\qons_t(w_t) 
    + \frac{\gamma}{2} (w_{1} - u)^\top A_1(w_{1} - u)
    \\
    &
    + \frac{\gamma}{2}\sum_{t=2}^T (w_{t} - u)^\top (A_t - A_{t-1})(w_{t} - u)
    \\
    ={}&
    \frac{1}{2\gamma}\sum_{t=1}^T \nabla\qons_t(w_t)^\top A_t^{-1}\nabla\qons_t(w_t) 
    + \frac{\gamma}{2} (w_{1} - u)^\top (A_1 - \nabla\qons_1(w_1)\nabla\qons_1(w_1)^\top)(w_{1} - u)
    \\
    &
    + \frac{\gamma}{2}\sum_{t=1}^T (w_{t} - u)^\top \nabla\qons_t(w_t)\nabla\qons_t(w_t)^\top(w_{t} - u). 
  \end{align}
  Since we have $A_0 = \varepsilon I_n$ and $\varepsilon = \frac{n}{W^2\gamma^2}$, the above inequality implies
  \begin{align}\label{eq:quadratic-bound}
    \begin{aligned}
    &\sum_{t=1}^T \nabla\qons_t(w_t)^\top\prn*{w_t - u} 
    - 
    \frac{\gamma}{2}\sum_{t=1}^T (w_{t} - u)^\top \nabla\qons_t(w_t)\nabla\qons_t(w_t)^\top(w_{t} - u)
    \\
    \le{}& 
    \frac{1}{2\gamma}\sum_{t=1}^T \nabla\qons_t(w_t)^\top A_t^{-1}\nabla\qons_t(w_t) 
    + \frac{\gamma}{2} (w_{1} - u)^\top (A_1 - \nabla\qons_1(w_1)\nabla\qons_1(w_1)^\top)(w_{1} - u) 
    \\
    \le{}& 
    \frac{1}{2\gamma}\sum_{t=1}^T \nabla\qons_t(w_t)^\top A_t^{-1}\nabla\qons_t(w_t) 
    + \frac{\gamma\varepsilon}{2} \norm{w_{1} - u}_2^2 
    \\
    \le{}&
    \frac{1}{2\gamma}\sum_{t=1}^T \nabla\qons_t(w_t)^\top A_t^{-1}\nabla\qons_t(w_t) 
    + \frac{n}{2\gamma}.
    \end{aligned}
  \end{align}
  The first term in the right-hand side is bounded as follows due to the celebrated elliptical potential lemma (e.g.,~\citet[proof of Theorem~4.5]{Hazan2023-pq}):
  \begin{equation}\label{eq:elliptic-potential-lemma}
    \sum_{t=1}^T \nabla\qons_t(w_t)^\top A_t^{-1}\nabla\qons_t(w_t)
    \le 
    \ln \frac{\det A_T}{\det A_0} 
    \le 
    n\ln \prn*{\frac{T\lambda^2}{\varepsilon} + 1}
    =
    n\ln \prn*{\frac{W^2\gamma^2\lambda^2T}{n} + 1},
  \end{equation}
  where we used $\det A_0 = \varepsilon^n$ and $\det A_T = \det\prn*{\sum_{t=1}^T \nabla\qons_t(w_t)\nabla\qons_t(w_t)^\top + \varepsilon I_n} \le \prn*{T\lambda^2 + \varepsilon}^n$.

  Combining \eqref{eq:regret-quadratic}, \eqref{eq:quadratic-bound}, and \eqref{eq:elliptic-potential-lemma}, we obtain
  \begin{align}
    \sum_{t=1}^T \prn*{\qons_t(w_t) - \qons_t(u)}
    &\le 
    \sum_{t=1}^T \nabla\qons_t(w_t)^\top\prn*{w_t - u} 
    - 
    \frac{\gamma}{2}\sum_{t=1}^T (w_{t} - u)^\top \nabla\qons_t(w_t)\nabla\qons_t(w_t)^\top(w_{t} - u)
    \\
    &\le
    \frac{n}{2\gamma}
      \prn*{
        \ln \prn*{\frac{W^2\gamma^2\lambda^2T}{n} + 1} + 1
      }
  \end{align}  
  as desired.
\end{proof}

\subsection{Regret bound of $\eta$-expert}
We now establish the regret bound of ONS in \cref{prop:ons-regret-surrogate}, which is used by $\eta$-experts in MetaGrad. 
Let $\eta \in \left(0, \frac{1}{5H}\right]$ and consider applying ONS to the following loss functions, which are defined in \eqref{eq:surrogate-loss}:
\[
  \fmg^\eta_t(w) = - \eta\inpr{w_t - w, g_t} + \eta^2 \inpr{w_t - w, g_t}^2 
  \quad
  \text{for $t = 1,\dots,T$}.
\]
As in \cref{prop:ons-regret-surrogate}, the $\ell_2$-diameter of $\Wcal$ is at most $W > 0$, and the following conditions hold:\looseness=-1
\begin{equation}
  \begin{aligned}
    w_t \in \Wcal,
    &&
    \norm{g_t}_2 \le G,
    &&
    \text{and}
    &&
    \max_{w,w'\in\Wcal} \inpr{w - w', g_t} \le H  
    &&
    \text{for $t = 1,\dots,T$.}
  \end{aligned}
\end{equation}
From 
$
\nabla\fmg^\eta_t(w) = \eta \prn*{1 - 2\eta g_t^\top (w_t - w)} g_t
$ 
and 
$
\nabla^2\fmg^\eta_t(w) = 2\eta^2 g_t g_t^\top
$, we have 
\begin{align}
  &
  \begin{aligned}
    \nabla\fmg^\eta_t(w) \nabla\fmg^\eta_t(w)^\top
    &=
    \eta^2 \prn*{1 - 2\eta g_t^\top (w_t - w)}^2 g_t g_t^\top
    \\
    &\preceq
    \eta^2(1 + 2\eta H)^2 g_t g_t^\top
    =
    \frac{(1 + 2\eta H)^2}{2} \nabla^2 \fmg^\eta_t(w)
    \quad
    \text{for all $w \in \Wcal$}, 
  \end{aligned}
  \\
  &
  \begin{aligned}
    \max_{w \in \Wcal} 
    \abs*{
    \nabla \fmg^\eta_t(w^\eta_t)^\top (w - w^\eta_t) 
    }
    &=
    \max_{w \in \Wcal} 
    \abs*{
      \eta g_t^\top (w - w^\eta_t) - 2\eta^2 \prn*{g_t^\top (w^\eta_t - w_t)}^2
    }
    \\
    &\le
    \eta H + 2\eta^2 H^2,
  \end{aligned}
  \\
  &\norm*{\nabla\fmg^\eta_t(w)}_2 
  = 
  \norm*{
    \eta \prn*{1 - 2\eta g_t^\top (w_t - w)} g_t
  }_2
  \le
  \eta(1+2\eta H)G. 
\end{align}
Therefore, $\fmg^\eta_t$ satisfies the conditions in \cref{prop:ons-regret} with 
$\alpha = \frac{2}{(1 + 2\eta H)^2}$, 
$\beta = \eta H + 2\eta^2 H^2$, and 
$\lambda = \eta(1+2\eta H)G$.
Since $\frac1\alpha = \frac12 + 2\eta H + 2\eta^2H^2 \ge \beta$ holds, we have 
$\gamma = \frac12\min\set[\big]{\frac{1}{\beta}, \alpha} = \frac{\alpha}{2}$.
Thus, for any $\eta \in \left(0, \frac{1}{5H}\right]$, we have $\gamma \in \left[\frac{25}{49}, 1\right) \subseteq \left[\frac{1}{2}, 1\right]$ and $\gamma\lambda = \frac{\eta G}{1 + 2\eta H} \le \frac{G}{7H}$.
Consequently, \cref{prop:ons-regret} implies that for any $u \in \Wcal$, the regret of the $\eta$-expert's ONS is bounded as follows:
\begin{equation}\label{eq:expert-regret}
  \sum_{t=1}^T \prn*{\fmg^\eta_t(w^\eta_t) - \fmg^\eta_t(u)}
  \le 
  n\prn*{
    \ln \prn*{\frac{W^2G^2T}{49nH^2} + 1} + 1
  }
  =
  O 
  \prn*{
    n \ln \prn*{\frac{WGT}{Hn}}
  }
  .
\end{equation}

\begin{algorithm}[tb]
	\caption{MetaGrad}
	\label{alg:meta}
	\begin{algorithmic}[1]
    \State $p^{\eta_i}_1 \gets \frac{C}{(i+1)(i+2)}$ for all $\eta_i \in \mathcal{G} = \Set*{\frac{2^{-i}}{5H}}{i = 0,1,\dots,\ceil*{\frac12\log_2 T}}$.
    \For{$t=1,\dots,T$}
    \State Fetch $w^\eta_t$ from $\eta$-experts for all $\eta \in \mathcal{G}$.
    \State Play $w_t = \frac{\sum_{\eta\in\mathcal{G}} \eta p^\eta_t w^\eta_t}{\sum_{\eta\in\mathcal{G}} \eta p^\eta_t}$.
    \State Observe $g_t \in \partial \fmg_t(w_t)$ and send $(w_t, g_t)$ to $\eta$-experts for all $\eta \in \mathcal{G}$.
    \State $p^\eta_{t+1} \gets p^\eta_t \exp(-\fmg^\eta_t(w^\eta_t)) / Z_t$ for all $\eta \in\mathcal{G}$, where $Z_t = \sum_{\eta\in\mathcal{G}} p^\eta_t \exp(-\fmg^\eta_t(w^\eta_t))$.
    \EndFor
	\end{algorithmic}
\end{algorithm}

\subsection{Regret bound of MetaGrad}\label[appendix]{subsec:meta-algorithm}
We turn to MetaGrad applied to convex loss functions $\fmg_1,\dots,\fmg_T\colon\Wcal\to\R$. 
We here use $w_t \in \Wcal$ and $g_t \in \partial\fmg_t(w_t)$ to denote the $t$th output of MetaGrad and a subgradient of $\fmg_t$ at $w_t$, respectively, for $t = 1,\dots,T$.
We assume that these satisfy the conditions in \eqref{eq:gt-condition}, as stated in \cref{prop:metagrad-regret}.

\Cref{alg:meta} describes the procedure of MetaGrad.
Define $\eta_i = \frac{2^{-i}}{5H}$ for $i = 0,1,\dots,\ceil*{\frac12\log_2 T}$, called \emph{grid points}, and let $\mathcal{G} \subseteq \left(0, \frac{1}{5H}\right]$ denote the set of all grid points.
For each $\eta \in \mathcal{G}$, $\eta$-expert runs ONS with loss functions $\fmg^\eta_1,\dots,\fmg^\eta_T$ to compute $w^\eta_1,\dots,w^\eta_T$. 
In each round $t$, we obtain~$w_t$ by aggregating the $\eta$-experts' outputs $w^\eta_t$ based on the exponentially weighted average method (EWA). 
We set the prior as $p^{\eta_i}_1 = \frac{C}{(i+1)(i+2)}$ for all $\eta_i \in \mathcal{G}$, where $C = 1 + \frac{1}{1 + \ceil*{\frac12\log_2 T}}$.
Then, it is known that for every $\eta \in \mathcal{G}$, the regret of EWA relative to the $\eta$-expert's choice $w^\eta_t$ is bounded as follows:
\begin{equation}\label{eq:meta-regret}
  \begin{aligned}
    \sum_{t=1}^T \prn*{\fmg^\eta_t(w_t) - \fmg^\eta_t(w^\eta_t)} 
    \le 
    \ln \frac{1}{p^\eta_1}
    &\le 
    \ln\prn*{
      \prn*{
        \ceil*{\frac12\log_2 T} + 1
      }
      \prn*{
        \ceil*{\frac12\log_2 T} + 2
      }
    }
    \\
    &\le
    2\ln\prn*{\frac12\log_2 T + 3},  
  \end{aligned}
\end{equation}
where we used $C \ge 1$ in the second inequality.
We here omit the proof as it is completely the same as that of \Citet[Lemma~4]{van-Erven2016-mg} (see also \citet[Lemma~1]{Wang2020-qv}).

We are ready to prove \cref{prop:metagrad-regret}.
Let $V_T^u = \sum_{t=1}^T \inpr{w_t - u, g_t}^2$.
By using $\fmg^\eta_t(w_t) = 0$,~\eqref{eq:expert-regret}, and \eqref{eq:meta-regret}, it holds that
\begin{align}
  \sum_{t=1}^T \inpr{w_t - u, g_t}
  &= 
  -\frac{\sum_{t=1}^T \fmg^\eta_t(u)}{\eta} + \eta V_T^{u}
  \\
  &= 
  \frac{1}{\eta}\prn[\Bigg]{
    \underbrace{\sum_{t=1}^T \prn*{\fmg^\eta_t(w_t) - \fmg^\eta_t(w^\eta_t)}}_{\text{Regret of EWA w.r.t.~$w^\eta_t$}} 
    + 
    \underbrace{\sum_{t=1}^T \prn*{\fmg^\eta_t(w^\eta_t) - \fmg^\eta_t(u)}}_{\text{Regret of $\eta$-expert w.r.t.~$u$}}} + \eta V_T^{u}
  \\
  &\le
  \frac{1}{\eta}
  \prn*{2\ln\prn*{\frac12\log_2 T + 3} + n\prn*{\ln \prn*{\frac{W^2G^2T}{49nH^2} + 1} + 1}}  
  + \eta V_T^{u} 
\end{align}
for all $\eta \in \mathcal{G}$.
For brevity, let 
\[
  A = 
  2\ln\prn*{\frac12\log_2 T + 3} + n\prn*{\ln \prn*{\frac{W^2G^2T}{49nH^2} + 1} + 1}
  \ge 1.
\]
If we knew $V^u_T$, we could set $\eta$ to $\eta^* \coloneqq \sqrt{\frac{A}{V_T^{u}}} \ge \frac{1}{5H\sqrt{T}}$ to minimize the above regret bound, $\frac{A}{\eta} + \eta V_T^{u}$.
Actually, we can do almost the same without knowing $V_T^{u}$ thanks to the fact that the regret bound holds for all $\eta \in \mathcal{G}$.
If $\eta^* \le \frac{1}{5H}$, by construction we have a grid point $\eta \in \mathcal{G}$ such that $\eta^* \in \brc*{\frac{\eta}{2}, \eta}$, hence\looseness=-1 
\[
  \sum_{t=1}^T \inpr{w_t - u, g_t} 
  \le 
  \eta V_T^{u} + \frac{A}{\eta} \le 2\eta^* V_T^{u} + \frac{A}{\eta^*} \le 3\sqrt{AV_T^{u}}.
\]
Otherwise, $\eta^* = \sqrt{\frac{A}{V_T^{u}}} \ge \frac{1}{5H}$ holds, which implies $V_T^{u} \le 25H^2A$. 
Thus, for $\eta_0 = \frac{1}{5H} \in \mathcal{G}$, we have
\[
  \sum_{t=1}^T \inpr{w_t - u, g_t} 
  \le 
  \eta_0 V_T^{u} + \frac{A}{\eta_0} \le 10HA.
\]
Therefore, in any case, we have
\[
  \sum_{t=1}^T \inpr{w_t - u, g_t} 
  \le 
  3\sqrt{AV_T^{u}} + 10HA
  =
  O\prn*{
    \sqrt{
      n \ln \prn*{\frac{WGT}{Hn}} 
      \cdot
      V_T^{u}
    }
    +
    Hn \ln \prn*{\frac{WGT}{Hn}} 
  },
\]
obtaining the regret bound in \cref{prop:metagrad-regret}.

\subsection{Lipschitz adaptivity and anytime guarantee}\label[appendix]{asubsec:lipschitz-adaptivity}
Recent studies \Citep{Mhammedi2019-iw,van-Erven2021-ji} have shown that MetaGrad can be further made Lipschitz adaptive and agnostic to the number of rounds. 
Specifically, MetaGrad given in \Citet[Algorithms~1 and~2]{van-Erven2021-ji} works without knowing $G$, $H$, or $T$ in advance, while using (a guess of) $W$. 
By expanding the proofs of \Citet[Theorem~7 and Corollary~8]{van-Erven2021-ji}, we can confirm that the refined version of MetaGrad enjoys the following regret bound:\looseness=-1
\[
  \sum_{t=1}^T \inpr{w_t - u, g_t}   
    =
    O\prn*{
    \sqrt{
      n \ln \prn*{\frac{WGT}{n}} 
      \cdot
      V_T^{u}
    }
    +
    Hn \ln \prn*{\frac{WGT}{n}} 
    }.
\]
By using this in the proof of \cref{thm:suboptimal-feedback}, we obtain 
\[
  \sum_{t=1}^T \inpr{c^*, x_t - \hat x_t}
  \le
  \sum_{t=1}^T \inpr{\hat c_t - c^*, \hat x_t - x_t}
  =
  O\prn*{
  Bn\ln\prn*{\frac{DKT}{n}}
  +
  \sqrt{
    \Delta_T Bn\ln\prn*{\frac{DKT}{n}}
  }
  },
\]
and the algorithm does not require knowing $K$, $B$, $T$, or $\Delta_T$ in advance.

\section{On removing the $\ln T$ factor: the case of $n=2$}\label[appendix]{asec:removing-lnT}
This section provides an additional discussion on closing the $\ln T$ gap in the upper and lower bounds on the regret.
Specifically, focusing on the case of $n=2$, we provide a simple algorithm that achieves a regret bound of $O(1)$ in expectation, removing the $\ln T$ factor.
We also observe that extending the algorithm to general $n \ge 2$ might be challenging. 
Note that \citet[Theorem~4.1]{Gollapudi2021-ad} has already established a regret bound of $\exp(O(n \ln n))$ as mentioned in \cref{subsec:related-work}, which implies an $O(1)$-regret bound for $n=2$.
The purpose of this section is simply to stimulate discussions on closing the $\ln T$ gap by presenting another simple analysis.
Below, let $\mathbb{B}^n$ and $\mathbb{S}^{n-1}$ denote the unit Euclidean ball and sphere in $\R^n$, respectively, for any integer $n > 1$.\looseness=-1

\subsection{An $O(1)$-regret method for $n=2$}\label[appendix]{asubsec:two-dim}
We focus on the case of $n=2$ and present an algorithm that achieves a regret bound of $O(1)$ in expectation.
We assume that all $x_t \in X_t$ are optimal for $c^*$ for $t=1,\dots,T$.
For simplicity, we additionally assume that all $X_t$ are contained in $\frac12\mathbb{B}^2$ and that $c^*$ lies in $\mathbb{S}^1$.
For any non-zero vectors $c, c' \in \R^n$, let $\theta(c,c')$ denote the angle between the two vectors. 
The following lemma from \citet{Besbes2023-zm}, which holds for general $n \ge 2$, is useful in the subsequent analysis.\looseness=-1
\begin{lemma}[{\citet[Lemma~1]{Besbes2023-zm}}]\label[lemma]{lem:sin-bound}
  Let $c^*, \hat c_t \in \mathbb{S}^{n-1}$, $X_t \subseteq \frac12\mathbb{B}^n$, $x_t \in \argmax_{x \in X_t}\inpr{c^*, x}$, and $\hat x_t \in \argmax_{x \in X_t}\inpr{\hat c_t, x}$. 
  If $\theta(c^*, \hat c_t) < \pi/2$, it holds that $\inpr{c^*, x_t - \hat x_t} \le \sin \theta(c^*, \hat c_t)$.
\end{lemma}

\begin{figure}
  \begin{minipage}[t]{0.5\textwidth}
    \begin{algorithm}[H]
      \caption{$O(1)$-Regret Algorithm for $n=2$.}
      \label{alg:policy}
      \begin{algorithmic}[1]
        \State Set $\Ccal_1$ to $\mathbb{S}^1$.
        \For{$t=1,\dots,T$}
        \State Draw $\hat c_t$ uniformly at random from $\Ccal_t$. 
        \State Observe $(X_t, x_t)$.
        \State $\Ccal_{t+1} \gets \Ccal_t \cap \Ncal_t$.
        \hfill{$\triangleright$ $\Ncal_t$ is the normal cone.}
        \EndFor
      \end{algorithmic}
    \end{algorithm}
  \end{minipage}%
  \hspace{.03\textwidth}
  \begin{minipage}[t]{0.45\textwidth}
    \begin{figure}[H]
    \centering
    \begin{tikzpicture}[scale=.7]
        \coordinate (O) at (0,0);
      
        \draw[black!70, line width=2pt] (20:2) arc (20:150:2);
        \draw[dashed, dash pattern=on 1pt off 1pt] (O) -- (20:2);
        \draw[dashed, dash pattern=on 1pt off 1pt] (O) -- (150:2);
        \node at (120:2.4) {\textcolor{black!80}{$\Ccal_t$}};
        
        \fill[black!60, opacity=.4] (O) -- (10:3) arc (10:100:3) -- cycle; 
        \node at (70:3.4) {\textcolor{black!80}{$\Ncal_t$}};

        \draw[black!100,line width=2pt] (20:2.1) arc (20:100:2.1);
        \draw[dashed, dash pattern=on 1pt off 1pt] (O) -- (20:2.1);
        \draw[dashed, dash pattern=on 1pt off 1pt] (O) -- (100:2.1);
        \node at (40:2.7) {\textcolor{black!100}{$\Ccal_{t+1}$}};
        \node at (69:2.5) {\textcolor{black!100}{$c^*$}};
        \draw[->, black!100, line width=1.5pt, dashed, dash pattern=on 3pt off 1pt] (O) -- (70:2.1);        
        \end{tikzpicture}
    \caption{Illustration of $c^*$, $\Ccal_t$, $\Ncal_t$, and $\Ccal_{t+1}$.}
    \label{fig:two-dim}
    \end{figure}
    
  \end{minipage}
\end{figure}
Our algorithm, given in \cref{alg:policy}, is a randomized variant of the one investigated by \citet{Besbes2021-ak,Besbes2023-zm}. 
The procedure is intuitive: 
we maintain a set $\Ccal_t \subseteq \mathbb{S}^1$ that contains $c^*$, from which we draw $\hat c_t$ uniformly at random, and update $\Ccal_t$ by excluding the area that is ensured not to contain~$c^*$ based on the $t$th feedback $(X_t, x_t)$.
Formally, the last step takes the intersection of $\Ccal_t$ and the \emph{normal cone} $\Ncal_t= \Set*{c \in \R^n}{\inpr{c, x_t - x} \ge 0, \forall x \in X_t}$ of $X_t$ at $x_t$, which is a convex cone containing~$c^*$. 
Therefore, every $\Ccal_t$ is a connected arc on $\mathbb{S}^1$ and is non-empty due to $c^* \in \Ccal_t$ (see \cref{fig:two-dim}).\looseness=-1

\begin{theorem}\label{thm:two-dim}
  For the above setting of $n=2$, \cref{alg:policy} achieves $\E\brc*{R^{c^*}_T} \le 2\pi$.
\end{theorem}
\begin{proof}
  For any connected arc $\Ccal \subseteq \mathbb{S}^1$, let $A(\Ccal) \in [0, 2\pi]$ denote its central angle, which equals its length. 
  Fix~$\Ccal_t$. 
  If $\hat c_t \in \Ccal_t \cap \mathrm{int}(\Ncal_t)$, where $\mathrm{int}(\cdot)$ denotes the interior, $\hat x_t = x_t$ is the unique optimal solution for $\hat c_t$, hence $\inpr{c^*, x_t - \hat x_t} = 0$. 
  Taking the expectation about the randomness of $\hat c_t$, we have
  \begin{align}
  \E\brc*{
    \inpr{c^*, x_t - \hat x_t}
  }
  &=
  \Pr\brc*{
    \hat c_t \in \Ccal_t \setminus \mathrm{int}(\Ncal_t)
  }
  \E\Brc*{
    \inpr{c^*, x_t - \hat x_t}
  }{
    \hat c_t \in \Ccal_t \setminus \mathrm{int}(\Ncal_t)
  }
  \\
  &=
  \frac{A(\Ccal_t \setminus \Ncal_t)}{A(\Ccal_t)}
  \E\Brc*{
    \inpr{c^*, x_t - \hat x_t}
  }{
    \hat c_t \in \Ccal_t \setminus \mathrm{int}(\Ncal_t)
  },
  \end{align}
  where 
  we used
  $
  \Pr\brc*{
    \hat c_t \in \Ccal_t \setminus \mathrm{int}(\Ncal_t)
  }
  =
  \Pr\brc*{
    \hat c_t \in \Ccal_t \setminus \Ncal_t
  }
  =
  A(\Ccal_t \setminus \Ncal_t)/A(\Ccal_t)
  $ 
  (since the boundary of~$\Ncal_t$ has zero measure).
  If $A(\Ccal_t) \ge \pi/2$,
  from  
  $\inpr{c^*, x_t - \hat x_t} \le \norm{c^*}_2\norm{x_t - \hat x_t}_2 \le 1$, we have
  \[
    \E\brc*{
      \inpr{c^*, x_t - \hat x_t}
    }
    \le 
    \frac{2}{\pi} A(\Ccal_t \setminus \Ncal_t) \le A(\Ccal_t \setminus \Ncal_t).
\]
If $A(\Ccal_t) < \pi/2$, \cref{lem:sin-bound} and $\hat c_t, c^* \in \Ccal_t$ imply $\inpr{c^*, x_t - \hat x_t} \le \sin \theta(c^*, \hat c_t) \le \sin A(\Ccal_t)$.
Thus, by using $\frac1x\sin x \le 1$ ($x \in \R$), we obtain 
\[
  \E\brc*{
    \inpr{c^*, x_t - \hat x_t}
  }
  \le 
  \frac{A(\Ccal_t \setminus \Ncal_t)}{A(\Ccal_t)} \sin A(\Ccal_t)
  \le A(\Ccal_t \setminus \Ncal_t).
\]
Therefore, we have $\E\brc*{\inpr{c^*, x_t - \hat x_t}} \le A(\Ccal_t \setminus \Ncal_t)$ in any case.
Consequently, we obtain 
\[
  \E\brc*{R^{c^*}_T}
  =
  \sum_{t=1}^T
  \E\brc*{\inpr{c^*, x_t - \hat x_t}} 
  \le
    \sum_{t=1}^T A(\Ccal_t \setminus \Ncal_t)
  \le 
  2\pi, 
\]
where the last inequality is due to $\Ccal_{t+1} = \Ccal_t \cap \Ncal_t$, which implies $\Ccal_s \subseteq \Ccal_t$ and $\Ccal_s \cap (\Ccal_t \setminus \Ncal_t) = \emptyset$ for any $s > t$, and hence no double counting occurs in the above summation.
\end{proof}

\subsection{Discussion on higher-dimensional cases}\label[appendix]{asubsec:higher-dim}
\Cref{alg:policy} might appear applicable to general $n \ge 2$ by replacing $\mathbb{S}^1$ with $\mathbb{S}^{n-1}$ and defining $A(\Ccal_t)$ as the area of $\Ccal_t \subseteq \mathbb{S}^{n-1}$. 
However, this idea faces a challenge in bounding the regret when extending the above proof to general $n \ge 2$.\footnote{
  We note that a hardness result given in \citet[Theorem~2]{Besbes2023-zm} is different from what we encounter here. 
  They showed that their \emph{greedy circumcenter policy} fails to achieve a sublinear regret, which stems from the shape of the initial knowledge set and the behavior of the greedy rule for selecting $\hat c_t$; this differs from the issue discussed above.
  }

\begin{figure}
    \begin{center}
    \tdplotsetmaincoords{70}{110}
  
    \pgfmathsetmacro{\rvec}{4}
    \pgfmathsetmacro{\thetavec}{40}
    \pgfmathsetmacro{\phivec}{45}
  
    \pgfmathsetmacro{\dphivec}{15}
  
    \pgfdeclarelayer{background}
    \pgfdeclarelayer{foreground}
  
    \pgfsetlayers{background, main, foreground}
  
    \begin{tikzpicture}[tdplot_main_coords, scale=.5, font=\fontsize{10}{10}\selectfont]
  
    \coordinate (O) at (0,0,0);
    \tdplotsetcoord{E}{\rvec}{\thetavec}{\phivec}
    \tdplotsetcoord{F'}{\rvec}{90}{\phivec}
    \tdplotsetcoord{G'}{\rvec}{90}{\phivec + \dphivec}
    \tdplotsetcoord{H}{\rvec}{\thetavec}{\phivec + \dphivec}
    \node at ($(E)+(0,0,0.5)$) {$c^*$};
    \filldraw[black!80] ($(E)$) circle (3pt);
    \node at ($(F')+(0,0,-0.6)$) {$\hat c_t$};
    \filldraw[black!80] ($(F')$) circle (3pt);
    \node at ($(F')+(E)+(0.2,0,-0.3)$) {$\Ccal_t$};
    \node at (1.4,0.8,0) {$\varepsilon$};
  
  
    \begin{pgfonlayer}{background}
      \shade[ball color=white, very thin, tdplot_screen_coords, opacity=0.6] (0,0) circle (\rvec);
      \draw [dashed, dash pattern=on 1pt off 1pt] (0,0,0) circle (\rvec) ;
    \end{pgfonlayer}
  
    \begin{pgfonlayer}{foreground}
      \draw[dashed, dash pattern=on 1pt off 1pt] (O) -- (E); 
      \draw[dashed, dash pattern=on 1pt off 1pt] (O) -- (F');
    \end{pgfonlayer}
    \begin{pgfonlayer}{background}
      \draw[dashed, dash pattern=on 1pt off 1pt] (O) -- (H);
      \draw[dashed, dash pattern=on 1pt off 1pt] (O) -- (G');
    \end{pgfonlayer}
  
    \begin{pgfonlayer}{foreground}
      \tdplotsetthetaplanecoords{\phivec}
      \tdplotdrawarc[]{(O)}{1.5}{\phivec}{\phivec + \dphivec}{}{}
    \end{pgfonlayer}
  
    \begin{pgfonlayer}{main}
      \fill[black!100, opacity=0.6] (E) to[bend left=8] (F')  to[bend left=-5] (G') to[bend right=9] (H) to[bend right=4] cycle;
    \end{pgfonlayer}
  
    \end{tikzpicture}
    \end{center}
    \caption{An example of $\Ccal_t$ on $\mathbb{S}^2$. 
    The darker area, $A(\Ccal_t)$, becomes arbitrarily small as $\varepsilon \to 0$, while $\theta(c^*, \hat c_t)$ does not.}
    \label{fig:elongated}
\end{figure}

As suggested in the proof of \cref{thm:two-dim}, bounding 
$
\E\brc*{
      \inpr{c^*, x_t\!-\!\hat x_t}
    }
$
is trickier when $A(\Ccal_t)$ is small (cf.~the case of $A(\Ccal_t) < \pi/2$). 
Luckily, when $n = 2$, we can bound it thanks to \cref{lem:sin-bound} and $\sin\theta(c^*, \hat c_t) \le \sin A(\Ccal_t)$, where the latter roughly means the angle, $\theta(c^*, \hat c_t)$, is bounded by the area, $A(\Ccal_t)$, from above. 
Importantly, when $n = 2$, both the central angle and the area of an arc are identified with the length of the arc, which is the key to establishing $\sin\theta(c^*, \hat c_t) \le \sin A(\Ccal_t)$. 
This is no longer true for $n \ge 3$. 
As in \cref{fig:elongated}, the area, $A(\Ccal_t)$, can be arbitrarily small even if the angle within there, or the maximum $\theta(c^*, \hat c_t)$ for $c^*, \hat c_t \in \Ccal_t$, is large.\footnote{
  A similar issue, though leading to different challenges, is noted in \citet[Section~4.4]{Besbes2023-zm}, where their method encounters ill-conditioned (or elongated) ellipsoids.
  They addressed this by appropriately determining when to update the ellipsoidal cone. 
  The $\ln T$ factor arises as a result of balancing being ill-conditioned with the instantaneous regret.
} 
This is why the proof for the case of $n=2$ does not directly extend to higher dimensions.
We leave closing the $O(\ln T)$ gap for $n \ge 3$ as an important open problem for future research.

\end{document}